\documentclass[11pt]{article}

\usepackage[preprint]{acl}

\usepackage{times}
\usepackage{latexsym}

\usepackage{amsmath,amsfonts,bm}

\def\secref#1{section~\ref{#1}}

\def\eqref#1{equation~\ref{#1}}

\def\1{\bm{1}}

\DeclareMathAlphabet{\mathsfit}{\encodingdefault}{\sfdefault}{m}{sl}
\SetMathAlphabet{\mathsfit}{bold}{\encodingdefault}{\sfdefault}{bx}{n}

\usepackage[T1]{fontenc}

\usepackage[utf8]{inputenc}

\usepackage{microtype}
\usepackage{xspace}
\usepackage{amsmath, amssymb}
\usepackage{graphicx}

\usepackage{inconsolata}

\usepackage{graphicx}

\usepackage[utf8]{inputenc} 
\usepackage[T1]{fontenc}    
\usepackage{hyperref}       
\usepackage{url}            
\usepackage{booktabs}       
\usepackage{amsfonts}       
\usepackage{nicefrac}       
\usepackage{microtype}      
\usepackage{xcolor}         
\usepackage{graphicx}
\usepackage{wrapfig}
\usepackage{amsmath}
\usepackage{natbib}
\usepackage{xspace}
\usepackage{amssymb}
\usepackage{float}
\usepackage[normalem]{ulem}
\usepackage{algorithm}
\usepackage{cleveref}
\usepackage{threeparttable}
\usepackage[normalem]{ulem}
\usepackage[most]{tcolorbox}
\usepackage{multirow}
\usepackage{makecell}
\usepackage{bm} 
\usepackage[most]{tcolorbox}
\tcbuselibrary{listingsutf8}
\usepackage{longtable}
\usepackage{array}
\usepackage{algpseudocode}
\usepackage{amsthm}
\usepackage[table]{xcolor}

\newcommand{\method}{MTCal\xspace}
\newcommand{\decodingname}{ConfChat\xspace}
\newcommand{\fakeparagraph}[1]
{\vspace{0.5mm}\noindent\textbf{#1}}

\newcommand{\ie}{i.e.,}

\newtheorem{theorem}{Theorem}

\usepackage{soul}

\title{Confidence Should Be Calibrated More Than One Turn Deep}

\author{Zhaohan Zhang\textsuperscript{$1,\ast$}, Chengzhengxu Li\textsuperscript{$2$}, Xiaoming Liu\textsuperscript{$2$},  Chao Shen\textsuperscript{$2$}, 
\\ {\bf Ziquan Liu\textsuperscript{$1$}}, {\bf Ioannis Patras\textsuperscript{$1$}}  \\
        \textsuperscript{1} Queen Mary University of London \quad
        \textsuperscript{2}Xi'an Jiaotong University \\ 
        \textsuperscript{$\ast$} Corresponding author \\
        \texttt{
        \{zhaohan.zhang, ziquan.liu, i.patras\}@qmul.ac.uk
        } \\
        \texttt{\{czx.li\}@stu.xjtu.edu.cn}, \texttt{\{chaoshen, xm.liu\}@xjtu.edu.cn}
        }

\begin{document}
\maketitle
\begin{abstract}
Large Language Models (LLMs) are increasingly applied in high-stakes domains such as finance, healthcare, and education, where reliable multi-turn interactions with users are essential. 
However, existing work on confidence estimation and calibration, a major approach to building trustworthy LLM systems, largely focuses on single-turn settings and overlooks the risks and potential of multi-turn conversations. 
In this work, we introduce the task of multi-turn calibration to reframe calibration from a static property into a dynamic challenge central to reliable multi-turn conversation, where calibrating model confidence at each turn conditioned on the conversation history is required.
We first reveal the risks of this setting: using Expected Calibration Error at turn T (ECE@T), a new metric that tracks calibration dynamics over turns, we show that user feedback (e.g., persuasion) can degrade multi-turn calibration.
To address this, we propose \method, which minimises ECE@T via a surrogate calibration target, and further leverage calibrated confidence in \decodingname, a decoding strategy that improves both factuality and consistency of the model response in multi-turn interactions.
Extensive experiments demonstrate that \method achieves outstanding and consistent performance in multi-turn calibration, and \decodingname preserves and even enhances model performance in multi-turn interactions. Our results mark multi-turn calibration as one missing link for scaling LLM calibration toward safe, reliable, and real-world use.

\end{abstract}

\section{Introduction}

The Large Language Models (LLMs) \cite{liu2024deepseek, dubey2024llama,yang2025qwen3,comanici2025gemini} are becoming indispensable assistants in interactive systems for real-world applications, especially in high-stakes domains such as finance \cite{xie2024finben}, medical \cite{li2024mediq, fan2025ai}, and education \cite{puech2024towards,liu2025one}.
Despite their impressive performance, there remain concerns about hallucinated or misleading outputs in multi-turn conversations.
Confidence estimation and calibration provide a promising solution to evaluating the reliability of model outputs by eliciting confidence scores that better align with empirical accuracy of the model responses \cite{kadavath2022language, tian2023just, ulmer2024calibrating, zhang2025reinforcement}.

\begin{figure}[!t]
    \centering
    \resizebox{0.5\textwidth}{!}{%
        \includegraphics{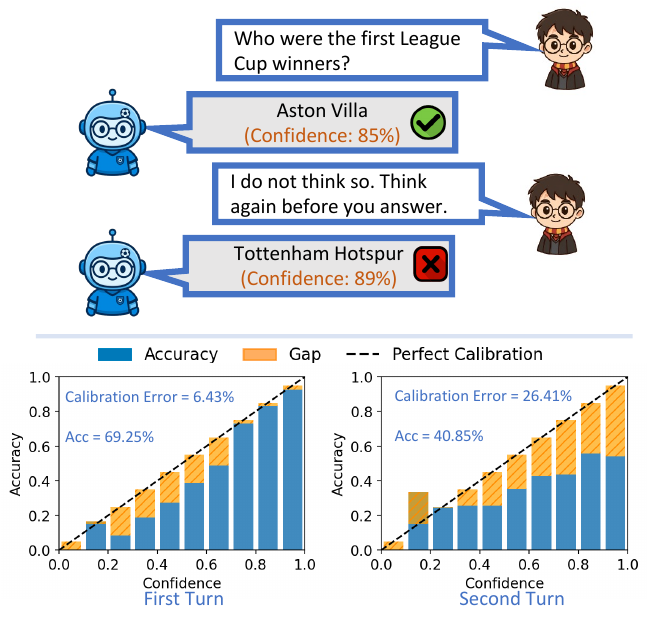} 
    }
    \vspace{-0.5cm}
	\caption{\textbf{LLMs are prone to change their responses with confidence when challenged.}
    The figure in the bottom left is the reliability diagram for confidence at the first turn. The figure in the bottom right is the reliability diagram for confidence at the second turn\protect\footnotemark.}
    \vspace{-0.5cm}
    \label{fig:intro} 
\end{figure}

\footnotetext{The experiment is conducted with Llama3.1-8B-instruct on TriviaQA dataset \cite{joshi2017triviaqa}.
We use the length-normalized likelihood of generated
sequences as the confidence measure.}

However, prior works on confidence estimation and calibration have primarily focused on single-turn settings, leaving multi-turn conversation scenarios largely unexplored. 
Meanwhile, recent studies suggest that self-generated context obtained through mechanisms such as self-reflection \cite{zhao2024fact} or extended reasoning \cite{mei2025reasoning, yoon2025reasoning}, can enhance confidence calibration.
Motivated by the practical significance of multi-turn interactions and evidence that richer context improves calibration, we introduce a practical and challenging task, multi-turn calibration, which requires the model to keep calibrated at every conversation turn with previous conversation history and ask: \textit{will multi-turn conversation history, which records multiple instances of model behavior within a conversation, improve multi-turn calibration?}
Addressing this question is crucial, as LLMs are known to be vulnerable to misleading user feedback such as external persuasion \cite{xu2024earth, stengel2025teaching}, conformity \cite{zhu2024conformity, cho2025herd}, and critique \cite{li2025firm}, which leads to behavioral shifts even when user feedback conflicts with the model’s internal knowledge.

To get started, we evaluate whether the model achieves multi-turn calibration by itself in adversarial interactions where a reliable confidence measure is essential.
As shown in Fig.~\ref{fig:intro}, we first pose a question to the model and attempt to persuade it to revise its belief in the subsequent interaction.
Frustratingly, the model becomes markedly overconfident at the second turn, \ie the predicted confidence greatly exceeds the corresponding empirical accuracy, indicating the unreliability of confidence in multi-turn interaction.
Confronted with the inability of LLMs to effectively leverage conversation history for multi-turn calibration, we propose \method, which introduces an auxiliary model as a calibrator for multi-turn calibration.
\method is designed to optimize confidence estimation with the objective of minimizing ECE@T.
Specifically, we train a Multilayer Perceptron (MLP)
to extract confidence from the model's hidden state using a surrogate calibration objective since ECE@T is non-differentiable. 
In addition, we propose a decoding strategy \decodingname that leverages calibrated confidence to improve the robustness against persuasion in multi-turn interactions.
\decodingname modifies the generation scores of candidate tokens at each turn by incorporating calibrated confidence, and then combines adjusted scores with the scores from the initial turn.
This aggregation, applied after the first turn, guides the model’s decision on whether to revise its response or preserve the original prediction.

Our contributions are summarized as follows:
\begin{itemize}
    \item We introduce the task of multi-turn calibration that uses conversation history for calibration in every conversation turn, along with an evaluation metric ECE@T.
    Our findings pose the risk that user feedback can be misleading to LLMs and degrade multi-turn calibration.

    \item We propose \method for improving multi-turn calibration by training a lightweight auxiliary model with surrogate calibration targets for minimizing ECE@T and design a confidence-based strategy \decodingname to improve the model performance and robustness in multi-turn conversations.
    
    \item Extensive experiments demonstrate that \method provides well-calibarted confidence at each conversation turn with ECE@T consistently under 10.0\% and \decodingname helps improve the response accuracy in persuasive interactions.
    
\end{itemize}

\section{Related Works}
\fakeparagraph{Multi-Turn Conversation.}
Multi-turn conversation is a common real-world application of LLMs where meaningful interactions occur through continuous changes of opinions \cite{chensteering, liang2024mathchat, qiu2024smile, li2025beyond}.
However, LLMs are found to shift their stance easily during multi-turn conversations \cite{sirdeshmukh2025multichallenge}.
This tendency is linked to sycophancy \cite{perez2023discovering}, where models cater to users’ opinions at the expense of factual accuracy, leading to inconsistent responses in multi-turn interactions when users disagree with the model’s initial belief \cite{xu2024earth, xie2024ask, li2025firm}.
\citet{xu2024earth, li2025firm} study such inconsistencies from the perspective of confidence, suggesting that the confidence score can serve as a proxy for response correctness.
Yet, it remains underexplored whether the model confidence keeps calibrated during conversations.



\fakeparagraph{Calibration.}
Calibration requires that the predicted confidence aligns with the empirical accuracy of the corresponding predictions \cite{guo2017calibration}, thereby enhancing the trustworthiness of LLM systems.
While previous studies suggest that demonstrations in in-context learning cause miscalibration in classification tasks \cite{zhoubatch, zhang2024study, li2025large}, recent works observe the potential for LLMs to utilize their own generation, such as reflection \cite{zhao2024fact, bodhwani2025calibrated, huang2025beyond} or extended reasoning \cite{tian2023just, mei2025reasoning, yoon2025reasoning}, to improve calibration in generative tasks.
In addition, auxiliary models have been explored to calibrate raw model confidence.
For example, \citet{kadavath2022language} introduces \textit{P}(True) which prompts the model to self-evaluate the correctness of its response via a True/False question.
\citet{kapoor2024large, huang2025efficient} train the model extensively to improve the calibration of \textit{P}(True).
Complementary to these methods, another line of work explores verbalized approaches, prompting or training models to articulate their confidence through natural language expressions \cite{tian2023just, hager2025uncertainty, li2025conftuner, zhang2025reinforcement}.
The ensumbling of different prompts and confidence estimations is also proven to be better calibrated \cite{jiang2023calibrating, xiongcan}.
Different from previous works on calibration which primarily target the single-turn question-answering tasks, our work proposes extending the calibration to multi-turn interaction scenarios where maintaining well-calibrated confidence across turns is essential.


\section{Problem Formulation}

\fakeparagraph{Multi-Turn Interaction.}
Let $\mathcal{M}$ be the language model.
The process by which $\mathcal{M}$ generates a response $r_t$ and gets the corresponding confidence $c_t$ in the $t$-th round of multi-turn interaction is as follows:
\begin{equation}
    (r_t, c_t) = f(\mathcal{M}, h_t),
\end{equation}
where $h_t$ is the conversation history the model receives at $t$-th turn, $f$ is a function for confidence estimation.
$h_t$ includes both interaction history from the previous $t-1$ turns and the user feedback in the $t$-th turn, formally:
\begin{equation}
h_{t} = \{(u_{1},r_{1}),\dots,(u_{t-1},r_{t-1}),u_{t}\}.
\end{equation}

\fakeparagraph{Single-Turn Calibration.}
We now define the single-turn calibration, which concerns how well confidence estimates reflect the true likelihood of correctness in a single round of conversation \cite{guo2017calibration}. 
Formally:
\begin{equation}
    \mathbb{P}(\sigma(r)=1|P=c) = c,
    \label{eq:global}
\end{equation}
where $P$ is the predicted probability of $r$ being correct, $\sigma(\cdot)$ is a binary function which returns 1 if $r$ is correct, and 0 otherwise.
Expected Calibration Error (ECE) is a widely used metric for empirically assessing how well confidence estimates are calibrated.
Given a dataset $\mathcal{D}_s=\{(u^{i},r^{i},c^{i})\}_{i=1}^{N}$, which consists of $N$ question-answer pairs along with their corresponding confidence scores in a single-turn interaction setting, ECE quantifies the expected difference between the predicted confidence and the true likelihood of correctness:
\begin{equation}
    \mathbb{E}_{\mathcal{D}_s}[|\mathbb{P}(\sigma(r)=1|P=c)-c].
\end{equation}
In practice, ECE is estimated as
\begin{equation}   \sum_{k=1}^{K}\frac{|\mathcal{B}_k|}{N}|\frac{1}{|\mathcal{B}_k|}\sum_{i\in \mathcal{B}_k}^{}\sigma (r^i)-\frac{1}{|\mathcal{B}_k|}\sum_{i\in \mathcal{B}_k}^{}c^i|, 
\end{equation}
where the question-answer pairs in $D$ are partitioned into $K$ bins of equal width, with $\mathcal{B}_k$ denoting the set of indices belonging to bin $k$. 

\fakeparagraph{Multi-Turn Calibration.}
We propose that in multi-turn interactions, the model should be calibrated at every turn for assessing the reliability of model responses at a finer granularity.
Specifically, given a multi-turn conversation dataset $\mathcal{D} = \{(h_t^i, r_t^i,c_t^i)|i=1,...,N, t=1,...,T_i\}$,
where $T_i$ is the number of total rounds of conversation $i$, the objective of multi-turn calibration is formally defined as:
\begin{equation}
    \forall t, \, \mathbb{P}(\sigma(r_t)=1|P=c_t) = c_t.
    \label{eq:mtcali}
\end{equation}
To evaluate the multi-turn calibration, we propose a new metric, ECE@T, to measure the model calibration at each fixed conversation turn $T$, formally:
\begin{equation}
\resizebox{\columnwidth}{!}{$
   \text{ECE@T}
      = \sum_{k=1}^{K}
        \frac{|\mathcal{B}_{Tk}|}{|\mathcal{D}_T|}
        \left|
           \frac{1}{|\mathcal{B}_{Tk}|}
           \sum\limits_{i\in \mathcal{B}_{Tk}}
              \sigma \bigl(r^{i}_T\bigr)
           -
           \frac{1}{|\mathcal{B}_{Tk}|}
           \sum\limits_{i\in \mathcal{B}_{Tk}}
              c_T^i
        \right|
$,}
\label{eq:ecet}
\end{equation}
where $\mathcal{D}_T = \{(h_t^i, r_t^i,c_t^i)|i=1,...,N_t, t=T\}$ is a subset of $\mathcal{D}$, $\mathcal{B}_{Tk}$ is set of indices belonging to bin $k$ at turn $T$.
In addition to turn-wise evaluation, we also define ECE@D over all conversation history-response pairs in $\mathcal{D}$ by grouping them into $K$ bins, providing a global measure of calibration across the entire multi-turn dataset:
\begin{equation}
\resizebox{\columnwidth}{!}{$
   \text{ECE@D}
      = \sum\limits_{k=1}^{K}
        \frac{|\mathcal{B}_k|}{|\mathcal{D}|}
        \left|
           \frac{1}{|\mathcal{B}_k|}
           \sum\limits_{(i,t)\in \mathcal{B}_k}
              \sigma \bigl(r^{i}_t\bigr)
           -
           \frac{1}{|\mathcal{B}_k|}
           \sum\limits_{(i,t)\in \mathcal{B}_k}
              c_t^i
        \right|
$.}
\label{eq:eceall}
\end{equation}


\begin{figure*}[t]
\centering

\includegraphics[width=0.85\textwidth]{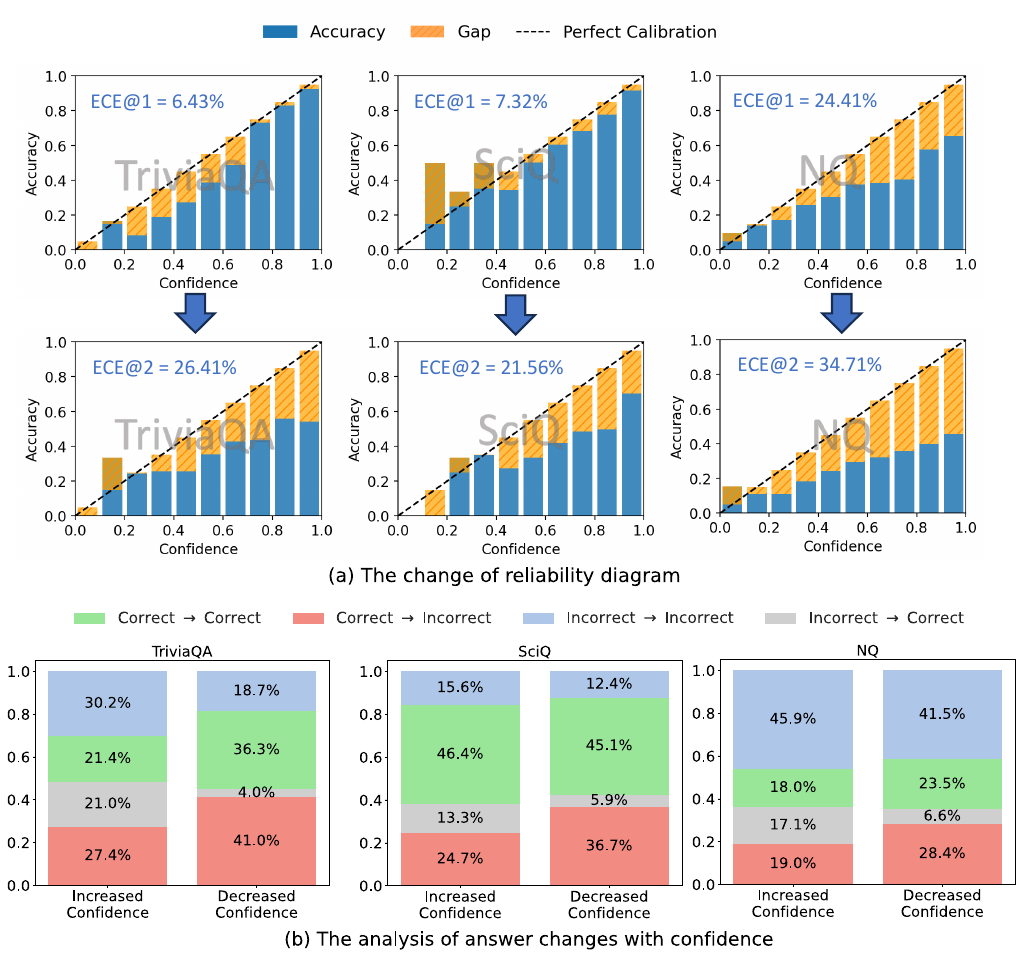}
\vspace{-0.3cm}
\caption{\textbf{(a) Changes in the reliability diagram from the initial response to the subsequent reply after receiving critical follow-up messages for Llama3.1-8B-Instruct. }
The diagrams above the arrows correspond to the first turn, while those below represent the second turn.
\textbf{(b) The analysis of answer changes with the change of model confidence.} 
Correct $\to$ Correct: The answer remains correct.
Correct $\to$ Incorrect: The correct answer is changed to an incorrect one.
Incorrect $\to$ Correct: The incorrect answer is revised to a correct one.
Incorrect $\to$ Correct: The answer still remains incorrect. 
A large portion of second turn responses with increased confidence comes with negative flips, \ie Correct $\to$ Incorrect.
 }
\vspace{-0.5cm}
\label{fig:eval}
\end{figure*}
\section{LLMs Fail to Use Conversation History for Multi-turn Calibration}
\label{sec:pilot}
In this section, we examine whether LLMs are able to utilize conversation history directly as context for improving multi-turn calibration.
If the model is truly trustworthy, it should stay well-calibrated even when it revises an initially correct answer under user pressure, reflecting appropriate confidence about the new response.
Specifically, we query the model with only questions at the first turn to obtain the initial answer.
In the second turn, we respond the model with messages randomly sampled from a set consist of messages with diverse persuasive strategies detailed in Appendix \ref{appd: followup} and track the change of the model’s calibration with ECE@T.
Common practices \cite{xu2024earth,li2025firm} take the likelihood of generating sequences $r$ as a confidence measure $c_s$:
\begin{equation}
    c_s
  = \exp\!\left(
      \frac{1}{\lvert r\rvert}
      \sum_{w \in r}
        \log p\bigl(w \mid \mathbf{w}_{<t}\bigr)
    \right),
    \label{eq:conf}
\end{equation}
where $p(\cdot)$ is the model predictive probability, $\mathbf{w}_{<t}$ represents the preceeding tokens.

We conduct experiments with Llama3.1-8B-Instruct \cite{dubey2024llama} on TriviaQA \cite{joshi2017triviaqa}, SciQ \cite{welbl2017crowdsourcing}, and NQ \cite{kwiatkowski2019natural}, and derive three key observations from the results presented in Fig.~\ref{fig:eval}.
\textbf{\textit{i}) Worse Calibration.} The model calibration becomes much worse when the LLM receive persuasive follow-up messages. 
The ECE@T increses by 19.98\%, 14.24\%, and  10.30\% on three benchmark datasets. 
This suggests that the confidence estimates at the second turn of the conversation become notably less reliable. 
Fig.~\ref{fig:eval}(a) reveals that the model becomes more overconfident, i.e., its predicted confidence exceeds the corresponding accuracy, at the second turn.
\textbf{\textit{ii}) Inconsistent Response.} The models are highly susceptible to being persuaded to abandon their initial correct beliefs in favor of incorrect ones.
As shown in Fig.~\ref{fig:eval}(b), there are far more responses that change from the correct answer to an incorrect one than vice versa.
The response inconsistency in multi-turn conversations underscores the importance of developing reliable confidence measures supporting model trustworthiness.
\textbf{\textit{iii}) Misleading Confidence Change.} 
Notably, 23.7\% of conversations with increased confidence shift from a correct initial answer to an incorrect one in average as reported in Fig.~\ref{fig:eval}(b).
The opposite change in response correctness and model confidence highlights the unreliability to evaluate model outputs based on model internal confidence in multi-turn interactions. 


\section{Multi-Turn Confidence Calibration}

\label{sec:mtcalibration}
To fully exploit the conversation history for improving multi-turn calibration, we propose \method which introduces an auxiliary model for probing confidence from the LLMs.

\fakeparagraph{Multi-Turn Calibration Objective.}
Traditional methods \cite{niculescu2005predicting} adopt binary labels as calibration targets, leading to a sharp distribution of the predicted confidence without a guarantee for calibration.
Ideally, the calibration procedure should aim to minimize ECE@T in multi-turn dialogues.
However, the calculation of ECE@T involves a binning operation and is non-differentiable, making it unsuitable to be a direct training objective.
Recalling Eq.~\ref{eq:ecet}, 
ECE@T measures the difference between the bin accuracy and the average confidence across turns.
Therefore, we propose to use turn-wise group accuracy as a surrogate calibration target and align it with the predicted confidence to minimize ECE@T.
Specifically, we group the model responses in each turn $t$ into $K$ bins with equal intervals according to the predicted confidence in Eq.~\ref{eq:conf}, and calculate the group accuracy as the calibration target:
    \begin{equation}
\operatorname{Acc}_{tk}
   = \frac{1}{|\mathcal{B}_{tk}|}
     \sum_{i \in \mathcal{B}_{tk}}
        \sigma\bigl(r^i_t\bigr),
\end{equation}
where $\mathcal{B}_{tk}$ is the set of indices assigned to bin $k$ at turn $t$.
Given the calibration target $\operatorname{Acc}_{tk}$, we define multi-turn calibration loss $\mathcal{L}_{MT}$ as:
\begin{equation}
    \mathcal{L}_{MT} = \frac{1}{N}\sum_{i=1}^{N}\frac{1}{T_i}\sum_{t=1}^{T_i}(\operatorname{Acc}_{tk}-c^i_t)^2,
\end{equation}
where $N$ is the number of conversation histories, $T_i$ is the number of conversation turns in the $i$-th conversation history. 

\fakeparagraph{Training Process.}
Inspired by the findings that the last hidden state of LLM encodes rich information about truthfulness \cite{li2023inference, liu2024enhancing, zhang2025reasoning}, we design a two-layer MLP as a light probe for estimating model confidence from the last hidden state $\bm{z}=[\bm{z}_1,...\bm{z}_M]$ without affecting the model's original ability.
Formally:
\begin{equation}
c = \bm{W}_2\!\bigl(\phi(\bm{W}_1\bar{\bm{z}}+\bm{b}_1)\bigr) + \bm{b}_2, \; 
\bar{\bm{z}} = \frac{1}{M}\sum_{i=1}^{M}\bm{z}_i,
\end{equation}
where $M$ is the number of input tokens, $\phi$ is the activation function.
We only optimize the probe with $\mathcal{L}_{MT}$ while keeping the language model frozen.

\section{\decodingname}

\begin{figure}[t]
  \centering
  \includegraphics[width=\columnwidth]{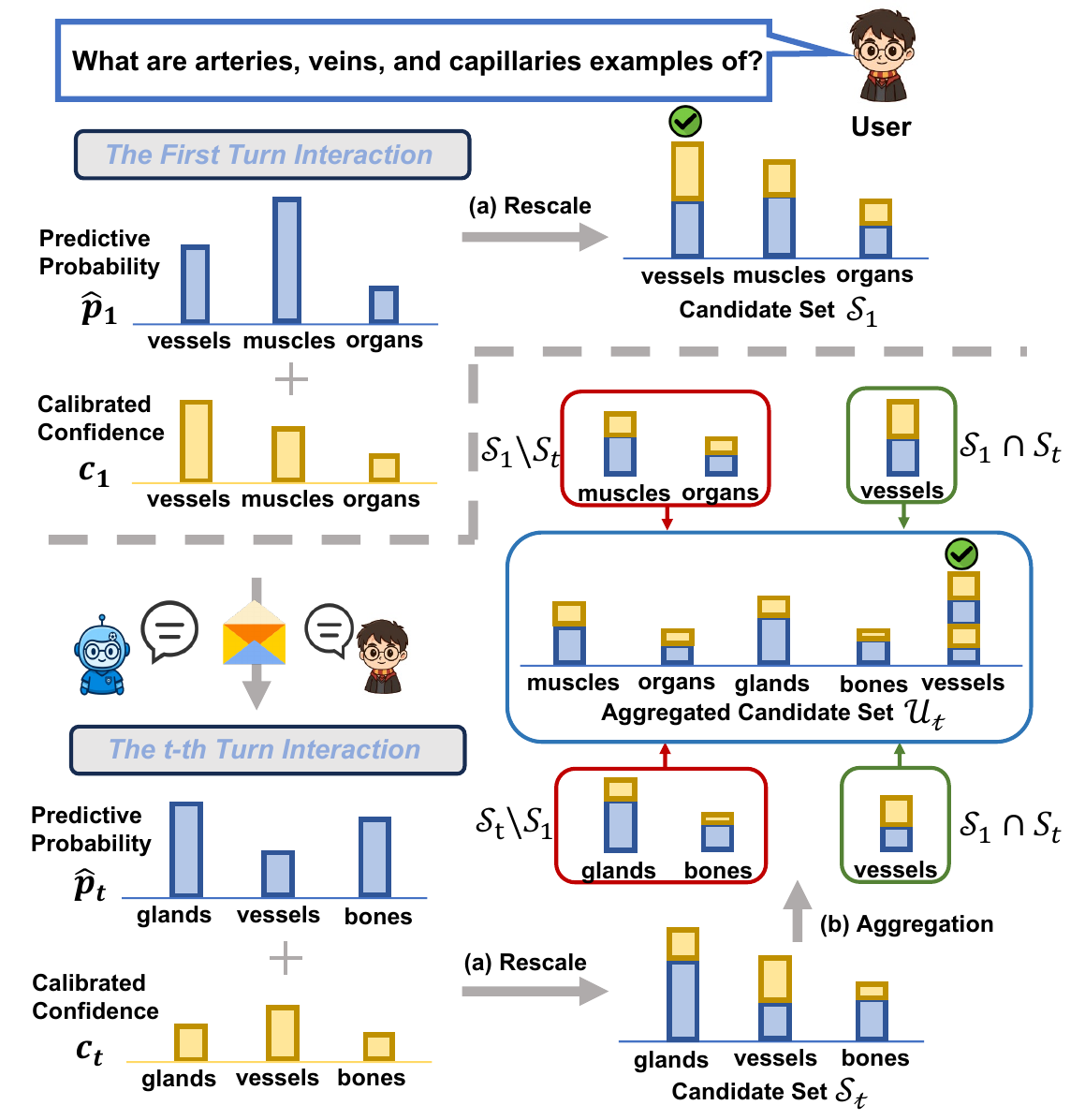}
  \caption{ 
    \textbf{The framework of \decodingname process. 
  }
  \textbf{(a) }In the first turn, the token with the highest rescaled generation score is selected at each decoding step.
\textbf{(b) }In subsequent turns, candidate token sets are generated based on both the first-turn and current-turn inputs, and the two candidate sets are aggregated to select the token with the highest overall generation score.
  }
  \vspace{-0.5cm}
  \label{fig:confchat} 
\end{figure}

\begin{table*}[!h]
\centering
\renewcommand{\arraystretch}{0.9}
\resizebox{\textwidth}{!}{
\begin{tabular}{cccccccccccccccc}

\toprule
\midrule
              & \multicolumn{5}{c}{\textbf{TriviaQA}} & \multicolumn{5}{c}{\textbf{SciQ}} & \multicolumn{5}{c}{\textbf{NQ}} \\
\cmidrule(lr){2-6} \cmidrule(lr){7-11} \cmidrule(lr){12-16}
 \textbf{Method}       & ECE@1 & ECE@2 & ECE@D & Brier & smECE & ECE@1 & ECE@2 & ECE@D & Brier & smECE & ECE@1 & ECE@2 & ECE@D & Brier & smECE 
\\ \specialrule{0.05em}{0.3em}{0.1em}
\rowcolor[gray]{0.9}
\multicolumn{16}{c}{\textbf{\textit{Llama3.1-8B-Instruct}}}
\\ \specialrule{0.05em}{0.1em}{0.3em}
 SL         & 6.43 & 26.41 & 9.36 & 22.17 & 9.27    & 7.32 & 21.56 & 12.05 & 23.43 &   10.51    & 24.12 & 34.97 & 27.54 & 31.31 & 24.41 \\
       PS&   12.93   &   26.44    &11.46 & 23.16 &   7.27  &10.03 & 13.36 & 12.50 & 21.47 &   12.67  &20.65 & 33.81 & 25.27 & 29.43 & 24.22 \\
       SC& 19.77 &  28.10 &25.52 & 28.88 &   20.67    &32.98 & 30.10 &38.71 & 39.10 & 29.38 & 20.94 & 21.79  & 22.82 & 28.64 & 17.29 \\
       Verbal       &21.96 & 41.72 &25.89 & 27.79 &  22.94     &15.12 & 32.09 &20.30 & 23.35 &  19.69     &41.67 & 51.12 &43.15 & 41.57 & 36.28      \\
       P(True)      &  22.87    &  32.47     &23.29 & 25.50 &  21.55     & 18.19     &   22.37    &17.91 & 21.27 &   17.17    &32.75 & 43.53 &41.05 & 40.25 &  34.46     \\
       \midrule
       DCal      &  6.60    &  21.29     &8.36 & 22.55 &  8.05     & 5.80     &   16.30    &6.23 & 21.81 &   6.20    &6.90 & 11.89 &7.61 & 23.58 &  7.65     \\
       \method        & \textbf{5.04} &  \textbf{2.31} & 
    \textbf{3.29} & \textbf{20.01} &  \textbf{2.94}     & \textbf{5.39} &  \textbf{6.39} & \textbf{5.65} & \textbf{20.83} &  \textbf{5.43}     & \textbf{5.18} &   \textbf{6.07}    & \textbf{6.03} & \textbf{22.78} & \textbf{5.15}      
\\ \specialrule{0.05em}{0.3em}{0.1em}
\rowcolor[gray]{0.9}
\multicolumn{16}{c}{\textbf{\textit{Qwen2.5-7B-Instruct}}}
\\ \specialrule{0.05em}{0.1em}{0.3em}
 SL         &24.10 & 33.38 &22.27 & 25.54 &   21.72 &12.09 & 24.16 &15.02 & 22.62 &   11.18 &34.11 & 49.51 &43.96 & 42.05 & 31.93 \\
       PS&17.22 & 28.89 &18.63 & 25.70 &   15.42    & 7.89 & 10.26 & 8.03 & 20.04 &   7.19 &32.02 & 40.09 &32.65 & 32.92 & 29.63 \\
       SC& 15.04 &  25.51 &18.82 & 22.90 &   12.94   &21.37 & 24.20 &23.52 & 27.80 & 18.74 & 25.03 &28.35&26.16 & 30.41 & 19.03 \\
       Verbal       &43.58 & 58.23 &39.44 & 39.81 &  31.10  &20.16 & 35.48 &24.36 & 23.63 &21.29&33.94 & 45.75 &43.49 & 43.91 & 39.91      \\
       P(True)      &31.38 & 38.27 &27.34 & 28.61 &  26.17  &17.94 & 21.99 &17.12 & 20.53 &  16.21  &40.54 & 56.30 &46.46 & 45.22 &  35.83     \\
       \midrule
       DCal      &  7.98    &  14.48     &9.91 & 22.74 &  8.71     & 5.46     &   14.52    &5.70 & 20.72 &   5.58    &13.55 & 17.59 &14.71 & 22.11 &  14.62 \\
       \method        & \textbf{7.72} &  \textbf{3.05} & \textbf{4.83} & \textbf{21.21} & \textbf{4.97}  & \textbf{4.78} &  \textbf{5.29} & \textbf{5.68} & \textbf{19.97} & \textbf{4.46}  & \textbf{6.22} &  \textbf{3.40} & \textbf{4.70} & \textbf{21.02} &  \textbf{3.71}     
\\ \specialrule{0.05em}{0.3em}{0.1em}
\rowcolor[gray]{0.9}
\multicolumn{16}{c}{\textbf{\textit{Gemma2-9B-it}}}
\\ \specialrule{0.05em}{0.1em}{0.3em}
 SL         &10.20 & 15.46 & 7.79 & 15.87 &   7.22    &15.02 & 21.08 &18.74 & 16.64 &  11.91     &22.43 & 26.15 &20.52 & 27.67 & 18.76 \\
       PS& 9.31 &  9.99 &11.32 & 16.61 & 10.39      &16.75 & 18.81 &17.40 & 15.43 & 12.87 &18.39 & 21.55 &18.93 & 25.22 & 17.15 \\
       SC& 30.71 & 36.22 &46.79 & 40.52 & 36.40   &44.27 & 47.41 &55.56 & 47.81 & 38.97 &20.43& 21.67  &28.93 & 32.78 & 22.83 \\
       Verbal       &26.41 & 30.92 &28.80 & 26.12 &23.05& 14.96 & 19.83 & 16.39 & 10.73 &8.65&48.11&47.39&40.74&40.58& 37.38      \\
       P(True)      &25.72 & 27.40 &17.58 & 17.81 &16.84&14.66 & 19.36 &13.89 & 14.64 &12.40&44.37 & 43.63 &38.21 & 37.50 &  31.22     \\
       \midrule
       DCal      &  14.67    &  16.42     &15.41 & 16.04 &  15.53     & 4.68     &   12.72    &5.33 & 12.79 &   5.40    &10.46 & 11.58 &14.59 & 23.85 &  13.98 \\
      \method        & \textbf{8.33} & \textbf{9.15} & \textbf{4.69} & \textbf{13.83} & \textbf{6.43}  & \textbf{4.42} &  \textbf{2.29} & \textbf{3.84} & \textbf{12.17} & \textbf{3.56}  & \textbf{4.68} &  \textbf{2.45} & \textbf{4.12} & \textbf{22.71} & \textbf{4.82}  \\
\midrule
\bottomrule
\end{tabular}
}
\caption{\textbf{The performance comparison of multi-turn calibration for different methods.}
The best results are \textbf{bolded}.
All the results are reported in percentage(\%).
We report ECE@1 and ECE@2 in the table, while the ECE@T for subsequent conversation turns is presented in Fig.~\ref{fig:ece_change}.
}
\vspace{-0.5cm}
\label{tab:exp}
\end{table*}

Building upon the well-calibrated confidence estimation from \method that faithfully reflects the likelihood that the model response is correct, we propose a decoding strategy \decodingname to guide the model to generate responses with high confidence during multi-turn conversation to improve the model's factuality and robustness to persuasion.
\decodingname adjusts the prediction scores of the top-$k$ candidate tokens using calibrated confidence obtained from \method.
Specifically, at each decoding step $i$ of conversation turn $t$, we obtain the probability distribution over the vocabulary $\mathcal{V}$ assigned by the language modeling head and select top-$k$ candidates $\mathbf{y}_{(i,t)}$ with the highest predictive probability $\hat{p}$.
For each candidate $y \in \mathbf{y}_{(i,t)}$ at turn $t$, we feed it to the language model and probe the corresponding confidence $c_t(y)$ with \method.
We combine the predictive probability on the candidates $\hat{p}_{t}{(y)}$ with $c_t(y)$ to get a rescaled generation score $s_t(y)$:
\begin{equation}
    s_t(y) = \lambda\hat{p}_t(y) + (1-\lambda)c_t(y),
\end{equation}
where $\lambda$ is a hyperparameter.
In the first turn, decoding is performed directly based on the rescaled scores $s_1(y)$, and the obtained $(y, s_1(y))$ pairs form a candidate set $\mathcal{S}_1$.
For each subsequent turn 
$t>1$, decoding is performed in a contextualised 
manner by conditioning the model on both the first and current turn inputs.
The candidate set $\mathcal{U}_t = \mathcal{S}_1 \cup  \mathcal{S}_t$ is obtained by merging the current candidates $\mathcal{S}_t$ with those from the first turn $\mathcal{S}_1$.
The final generation score $\tilde{s}_t(y)$ for each candidate $y\in\mathcal{U}_t$ is defined as:
\begin{equation}
\tilde{s}_t(y) = 
\begin{cases}
s_1(y) + s_t(y), & y \in \mathcal{S}_1 \cap \mathcal{S}_t, \\[6pt]
s_t(y), & y \in \mathcal{S}_t \setminus \mathcal{S}_1, \\[6pt]
s_1(y), & y \in \mathcal{S}_1 \setminus \mathcal{S}_t,
\end{cases}
\end{equation}
where overlapping candidates have their scores summed and candidates unique to either set retain their original rescaled scores.
We follow a greedy process and select the candidate with the highest generation score $\tilde{s}_t(y)$ as the final generation. 
In this way, we consider the model’s generations from both the first and current turns, enabling it to decide whether to maintain its confident initial response or to explore alternative options.


\vspace{-0.2cm}
\section{Experiments}
\label{sec:exp}
We conduct extensive experiments to evaluate the performance of \method in multi-turn calibration.

\vspace{-0.2cm}
\subsection{Multi-Turn Calibration}
\label{sec:mtc}
\vspace{-0.2cm}

\fakeparagraph{Datasets\&Models.}
We conduct experiments on three benchmark datasets: TriviaQA \cite{joshi2017triviaqa}, SciQ \cite{welbl2017crowdsourcing}, and NQ \cite{kwiatkowski2019natural}, respectively.
Three instruction-tuned models are tested: Llama3.1-8B-Instruct \cite{grattafiori2024llama}, Qwen2.5-7B-Instruct \cite{qwen2025qwen25technicalreport}, and Gemma2-9B-it \cite{team2024gemma} because they are optimized for following instructions and handling multi-turn conversations.
We employ \texttt{gpt-3.5-turbo} as an LLM-as-a-judge \cite{zheng2023judging} to evaluate the correctness of model responses.
To construct a multi-turn conversation dataset, we follow \cite{li2025firm} and reply to the model with messages randomly selected from Appendix \ref{appd: followup} at each turn following the query until the model's initial belief changes, as continuing beyond this point would no longer reflect the calibration of the same belief.


\fakeparagraph{Comparison Methods.}
We compare \method with five confidence estimation and calibration methods:
\textbf{Sequence Likelihood (SL)} \cite{malininuncertainty}, \textbf{Platt Scaling (PS)} \cite{platt1999probabilistic},
\textbf{Self-Consistency (SC)} \cite{xiongcan},
\textbf{Verbal} \cite{tian2023just}, and
\textbf{P(True)} \cite{kadavath2022language}.
Additionally, we include an ablation version of \method, \textbf{DCal}, which optimizes the auxiliary model for minimizing ECE@D for comparison.
For all methods except for PS, we take the conversation history $h_t$ and model response $r_t$ as input.
The detailed introduction is in Appendix \ref{appd:confidencemethods}.


\fakeparagraph{Metrics.}
We use \textbf{ECE@T} to track the change of calibration level during the conversations and \textbf{ECE@D} to reflect the global calibration level.
Moreover, we report two widely used calibration metrics to evaluate the overall calibration performance across all conversation history–answer pairs, including:
\textbf{Brier score} \cite{glenn1950verification}, which quantifies the mean squared difference between predicted probabilities and the actual label; and
\textbf{smECE} \cite{blasioksmooth}, a smoothed variant of ECE that provides a more stable and reliable estimation by mitigating the sensitivity to binning choices.
The smaller values of all metrics indicate better performance.

\fakeparagraph{Results.}
As shown in Table \ref{tab:exp}, all the baseline methods fail to maintain calibration during multi-turn interactions.
Although Llama3.1-8B-Instrct and Gemma2-9B-it exhibit better calibration than Qwen2.5-7B-Instruct, as indicated by their consistently smaller ECE@1 values across different confidence estimation methods, all models suffer from severe miscalibration as the dialogue progresses.
On average, the ECE@T increases by 10.14\%, 9.60\%, and 2.92\% from the first to the second turn across all comparison methods, for Llama3.1-8B-Instruct, Qwen2.5-7B-Instruct, and Gemma2-9B-it, respectively, reflecting the vulnerability of existing approaches to directly utilize conversation history for multi-turn calibration.
In contrast, \method exhibits notable stability and reliability throughout the conversations.
ECE@2 decreases by 1.03\% from ECE@1,
highlighting that \method serves as a reliable confidence measure for faithfully tracking model reliability under continued user interaction.
Furthermore, while DCal improves over baseline methods, it fails to guarantee calibration across all conversation turns, which demonstrates the effectiveness of the training objective in \method.
Our analysis further shows that the improvement on multi-turn calibration enhances the calibration across all question–answer pairs in the conversations.
We provided theoretical proof to this in Appendix \ref{appd:proof}.


\begin{figure}[!t]
    \centering
    \resizebox{0.5\textwidth}{!}{%
        \includegraphics{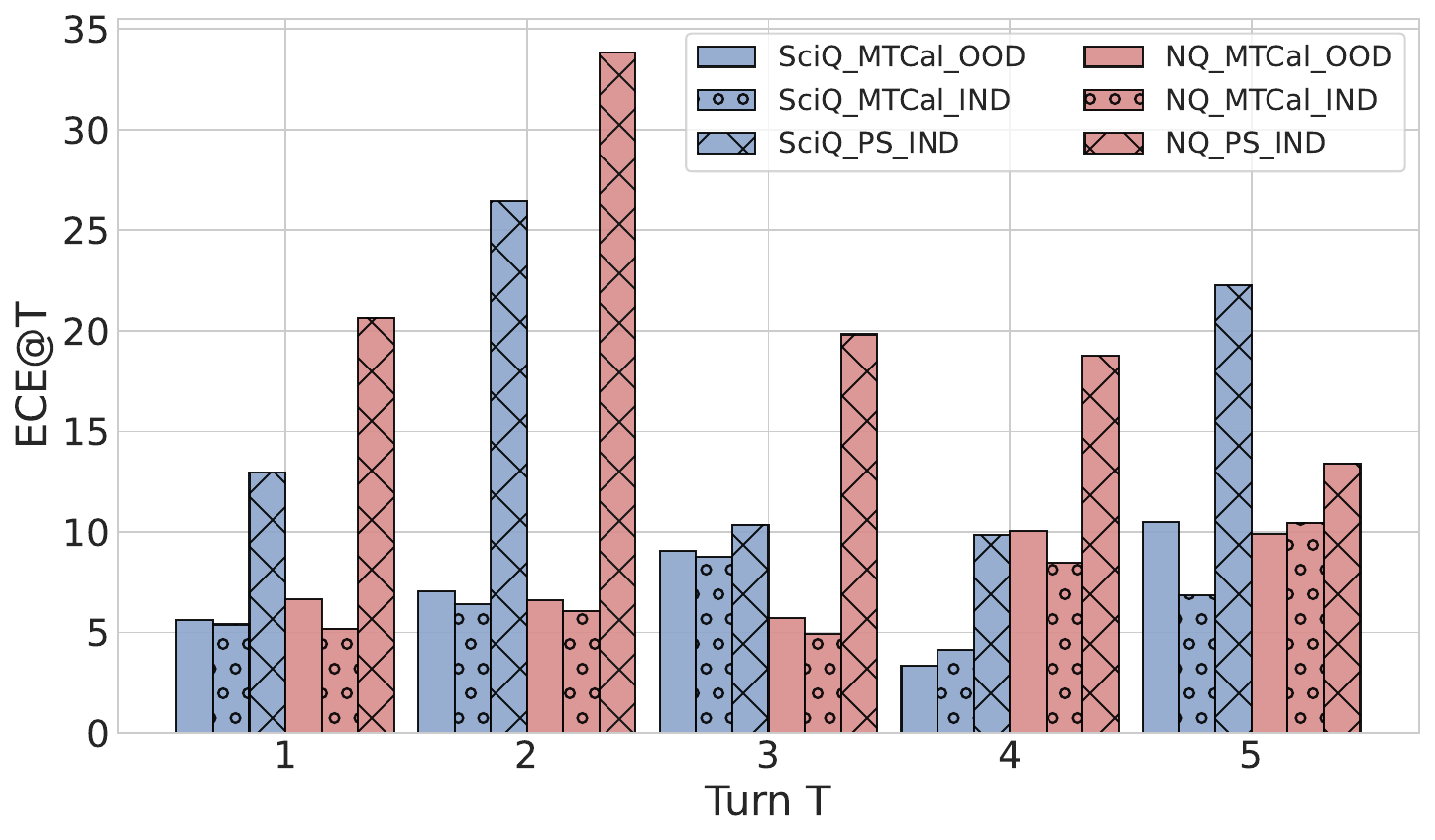} 
    }
    \caption{\textbf{Domain generalization on Llama-8B-Instruct.} 
    OOD denotes the out-of-domain setting, 
    IND denotes the in-domain setting, 
    and PS refers to Platt Scaling.
     }
    \vspace{-0.5cm}
    \label{fig:llama_dg} 
\end{figure}

\begin{figure*}[!t]
\centering

\includegraphics[width=\textwidth]{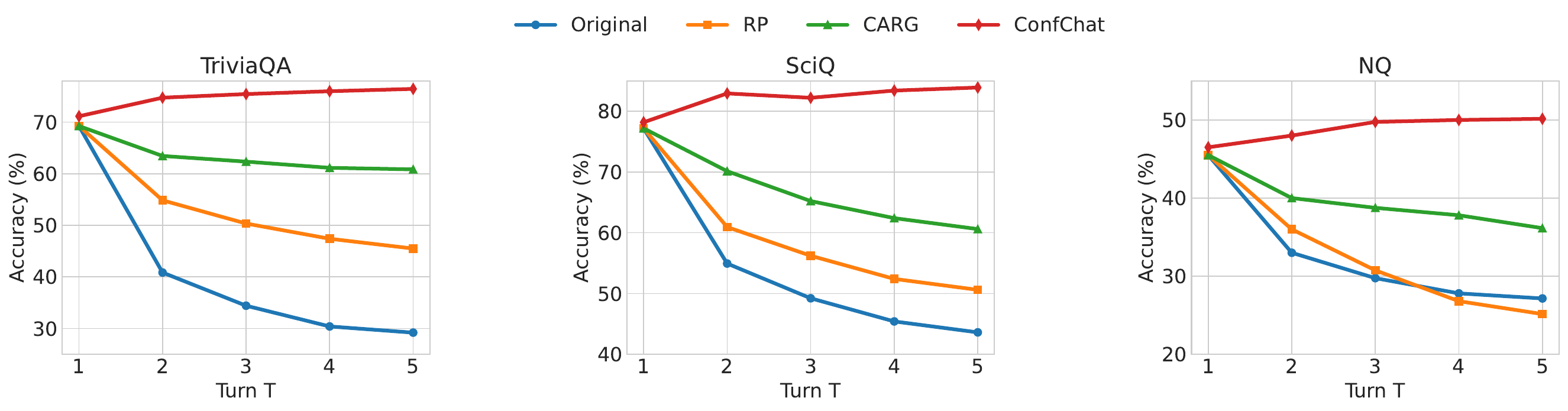}
\caption{\textbf{The comparison of change in accuracy of Llama3.1-8B-Instruct in different conversation rounds between \decodingname and other strategies.} Our method ConfChat keeps a relatively stable accuracy across turns.  }
\vspace{-0.1cm}
\label{fig:confchat}
\end{figure*}

\vspace{-0.2cm}
\subsection{Domain Generalization}

\method needs to train a probe on the calibration set, raising the question of whether it generalizes to other domains.
To investigate this, we use TriviaQA as the calibration set for \method and evaluate its multi-turn calibration performance on SciQ and NQ, representing out-of-domain settings.
We compare the out-of-domain results with in-domain performance where we train and test \method on the same dataset, and with a strong baseline, Platt Scaling, as reported in Fig.~\ref{fig:llama_dg} and Appendix \ref{appd: moreresults}.
Our results show that domain discrepancy has only a marginal effect on the multi-turn calibration of \method.
Moreover, \method substantially outperforms Platt Scaling even in out-of-domain scenarios.
It suggests that \method captures the model’s confidence in factuality rather than overfitting to the semantics or distribution of the calibration set, thereby demonstrating strong generalizability.

\subsection{Comparison with Single-turn Calibration}

In this section, we investigate whether multi-turn conversation histories provide useful signals for improving calibration by comparing \method with approaches designed for single-turn settings.
To enable a fair comparison, diffrent from the multi-turn calibration setting in \secref{sec:mtc}, we take the initial query $u_1$ and the response in the current turn $r_t$ as input for single-turn calibration methods.
We compare the performance of \method in multi-turn calibration with two strong single-turn calibration methods, SL and Apricot \cite{ulmer2024calibrating}.
ApriCot calibrates confidence by training a DeBERTa model \cite{hedebertav3} on semantic cluster accuracies, but input length limits and the complexity of conversation history prevent its direct application to multi-turn calibration.

We observe from the results in Table~\ref{tab:singleturn} that the performance of \method is comparable to single-turn methods at the first turn, but generally surpasses them in later turns of the conversation.
This result indicates that \method effectively leverages historical responses to calibrate confidence in subsequent turns. 
A case study provided in Appendix \ref{appd:case} further illustrates that patterns in conversation history (e.g., consistent responses) influence the confidence predictions of \method.

\begin{table}[!t]
\centering
\resizebox{\linewidth}{!}{
\begin{tabular}{lccccc}
\toprule
\midrule
 \textbf{Method} & \textbf{ECE@1} & \textbf{ECE@2} & \textbf{ECE@3} & \textbf{ECE@4} & \textbf{ECE@5}\\
\specialrule{0.05em}{0.3em}{0.1em}
\rowcolor[gray]{0.9}
\multicolumn{6}{c}{\textbf{\textit{Llama3.1-8B-Instruct}}}
\\ \specialrule{0.05em}{0.1em}{0.3em}
 SL  & 7.56   & 8.13  & 8.25 & 7.00 & 9.42 \\
Apricot  & 6.03  & 5.18 & 7.12 & 7.33 & 6.26\\
\method & \textbf{5.04}   & \textbf{2.31} & \textbf{5.91} & \textbf{5.01} & \textbf{6.08}\\
\specialrule{0.05em}{0.3em}{0.1em}
\rowcolor[gray]{0.9}
\multicolumn{6}{c}{\textbf{\textit{Qwen2.5-7B-Instruct}}}
\\ \specialrule{0.05em}{0.1em}{0.3em}
 SL  & 23.58 & 24.65 & 23.20 & 20.38 & 22.89\\ 
 Apricot  & 10.77   & 11.69 & 10.00 & 11.13 & 11.80\\
 \method & \textbf{7.72}  & \textbf{3.05} & \textbf{5.32} & \textbf{4.00} & \textbf{4.30}\\
\specialrule{0.05em}{0.3em}{0.1em}
\rowcolor[gray]{0.9}
\multicolumn{6}{c}{\textbf{\textit{Gemma2-9B-it}}}
\\ \specialrule{0.05em}{0.1em}{0.3em}
 SL  & \textbf{9.17} & 10.28 & 10.95 & 11.34 &10.66\\ 
  Apricot  &  10.14   & \textbf{9.22} & 9.71 & 8.30 & 9.87\\
 \method & 9.33  & 10.15 & \textbf{1.59} & \textbf{3.20} & \textbf{2.08}\\
\midrule
\bottomrule
\end{tabular}
}
\caption{\textbf{The comparison between single-turn calibration methods and \method.} All the results are reported in percentage(\%).
The best results are \textbf{bolded}\protect\footnotemark.
}
\vspace{-0.5cm}
\label{tab:singleturn}
\end{table}

\footnotetext{We calculate the ECE@T in the single-turn calibration setting on question-answer pair set consisting of initial query $u_1$ and model response $r_T$ at $T$-th turn.}

\subsection{\decodingname Improves Factuality in Multi-turn Conversation}
We compare \decodingname with two strategies that enhance model robustness against persuasive user feedback, \textbf{Reminder Prompt (RP)} \cite{xu2024earth} and \textbf{Confidence-Aware Response Generation (CARG)} \cite{li2025firm}.
The details of the comparison method are in Appendix \ref{appd: persuationmethods}.

We present the results on Llama3.1-8B-Instruct in Fig.~\ref{fig:confchat} and results on other models in Appendix \ref{appd: moreresults}. 
We find that informing the model of its confidence is more effective than instructing it to be cautious about user feedback, as evidenced by both confidence-based methods (CARG and \decodingname) outperforming RP.
However, CARG only appends the confidence score to the reply as additional context for the subsequent conversation, making its influence on model behavior indirect and less effective.
In contrast, \decodingname directly incorporates calibrated confidence into generation, improving response accuracy by 1.30\% at the first turn on average and further enhances model performance in subsequent interactions by leveraging decisions from the initial turn, suggesting that incorporating calibrated confidence into generation provides a decision-making principle that favors factuality.

\vspace{-0.1cm}
\section{Conclusion}
\vspace{-0.3cm}
We introduce the task of multi-turn calibration, which requires calibrating model confidence at each turn by leveraging the conversation history as prior, together with a new metric, ECE@T, to track calibration throughout the dialogue.
Through experiments in multi-turn persuasion scenarios, we find that model confidence calibration deteriorates, suggesting that conversation history can be misleading and degrade reliability.
We address this by developing \method, 
an auxiliary probe trained to minimize ECE@T with surrogate calibration targets. 
In addition, we design a strategy, \decodingname, which integrates calibrated confidence from \method into the generation process to enhance factuality and robustness against misleading user feedback. Extensive experiments demonstrate the effectiveness of \method and \decodingname in leveraging conversation history to improve the reliability of LLMs in multi-turn interactions.

\section*{Limitations}
Although \method improves multi-turn calibration and \decodingname enhances robustness against user persuasion, our work still has several limitations.
First, \method requires white-box access to extract hidden states, which limits its applicability to closed-source models such as the OpenAI GPT series.
Second, \method only provides a single confidence about the response factuality in every turn.
Evaluating the confidence in long-form generation that includes multiple claims is out of the scope of \method.
Third, \decodingname estimates confidence across $k$ candidates at each generation step, which improves robustness in multi-turn interactions but comes at the cost of efficiency.

\section*{Ethics Statement}
Our work focus on improving the reliability of large language models.
While the persuasion strategies used in this work are effective, we discourage any malicious use of our work, especially attempts to compromise LLM systems.
The artifacts and datasets in our work are all under the restriction of the license and follow the intended use.
We used GPT-5 as an AI writing assistant to refine and improve the clarity of our text.
All AI-generated suggestions were carefully reviewed and edited by the authors to ensure the integrity of the work.
The final manuscript reflects the authors’ original contributions, with AI assistance limited solely to enhancing the presentation of our findings.

\bibliography{custom}

@article{yang2025qwen3,
  title={Qwen3 technical report},
  author={Yang, An and Li, Anfeng and Yang, Baosong and Zhang, Beichen and Hui, Binyuan and Zheng, Bo and Yu, Bowen and Gao, Chang and Huang, Chengen and Lv, Chenxu and others},
  journal={arXiv preprint arXiv:2505.09388},
  year={2025}
}

@article{comanici2025gemini,
  title={Gemini 2.5: Pushing the frontier with advanced reasoning, multimodality, long context, and next generation agentic capabilities},
  author={Comanici, Gheorghe and Bieber, Eric and Schaekermann, Mike and Pasupat, Ice and Sachdeva, Noveen and Dhillon, Inderjit and Blistein, Marcel and Ram, Ori and Zhang, Dan and Rosen, Evan and others},
  journal={arXiv preprint arXiv:2507.06261},
  year={2025}
}

@article{dubey2024llama,
  title={The llama 3 herd of models},
  author={Dubey, Abhimanyu and Jauhri, Abhinav and Pandey, Abhinav and Kadian, Abhishek and Al-Dahle, Ahmad and Letman, Aiesha and Mathur, Akhil and Schelten, Alan and Yang, Amy and Fan, Angela and others},
  journal={arXiv e-prints},
  pages={arXiv--2407},
  year={2024}
}

@article{liu2024deepseek,
  title={Deepseek-v3 technical report},
  author={Liu, Aixin and Feng, Bei and Xue, Bing and Wang, Bingxuan and Wu, Bochao and Lu, Chengda and Zhao, Chenggang and Deng, Chengqi and Zhang, Chenyu and Ruan, Chong and others},
  journal={arXiv preprint arXiv:2412.19437},
  year={2024}
}

@inproceedings{fan2025ai,
  title={AI Hospital: Benchmarking Large Language Models in a Multi-agent Medical Interaction Simulator},
  author={Fan, Zhihao and Wei, Lai and Tang, Jialong and Chen, Wei and Wang, Siyuan and Wei, Zhongyu and Huang, Fei},
  booktitle={COLING},
  year={2025}
}

@article{li2024mediq,
  title={Mediq: Question-asking llms and a benchmark for reliable interactive clinical reasoning},
  author={Li, Stella and Balachandran, Vidhisha and Feng, Shangbin and Ilgen, Jonathan and Pierson, Emma and Koh, Pang Wei W and Tsvetkov, Yulia},
  journal={Advances in Neural Information Processing Systems},
  volume={37},
  pages={28858--28888},
  year={2024}
}

@inproceedings{liu2025one,
  title={One size doesn't fit all: A personalized conversational tutoring agent for mathematics instruction},
  author={Liu, Ben and Zhang, Jihai and Lin, Fangquan and Jia, Xu and Peng, Min},
  booktitle={Companion Proceedings of the ACM on Web Conference 2025},
  pages={2401--2410},
  year={2025}
}

@article{puech2024towards,
  title={Towards the pedagogical steering of large language models for tutoring: A case study with modeling productive failure},
  author={Puech, Romain and Macina, Jakub and Chatain, Julia and Sachan, Mrinmaya and Kapur, Manu},
  journal={arXiv preprint arXiv:2410.03781},
  year={2024}
}

@article{xie2024finben,
  title={Finben: A holistic financial benchmark for large language models},
  author={Xie, Qianqian and Han, Weiguang and Chen, Zhengyu and Xiang, Ruoyu and Zhang, Xiao and He, Yueru and Xiao, Mengxi and Li, Dong and Dai, Yongfu and Feng, Duanyu and others},
  journal={Advances in Neural Information Processing Systems},
  volume={37},
  pages={95716--95743},
  year={2024}
}

@article{kadavath2022language,
  title={Language models (mostly) know what they know},
  author={Kadavath, Saurav and Conerly, Tom and Askell, Amanda and Henighan, Tom and Drain, Dawn and Perez, Ethan and Schiefer, Nicholas and Hatfield-Dodds, Zac and DasSarma, Nova and Tran-Johnson, Eli and others},
  journal={arXiv preprint arXiv:2207.05221},
  year={2022}
}

@article{kapoor2024large,
  title={Large language models must be taught to know what they don’t know},
  author={Kapoor, Sanyam and Gruver, Nate and Roberts, Manley and Collins, Katie and Pal, Arka and Bhatt, Umang and Weller, Adrian and Dooley, Samuel and Goldblum, Micah and Wilson, Andrew G},
  journal={Advances in Neural Information Processing Systems},
  volume={37},
  pages={85932--85972},
  year={2024}
}

@inproceedings{ulmer2024calibrating,
  title={Calibrating Large Language Models Using Their Generations Only},
  author={Ulmer, Dennis Thomas and Gubri, Martin and Lee, Hwaran and Yun, Sangdoo and Oh, Seong Joon},
  booktitle={Proceedings of the 62nd Annual Meeting of the Association for Computational Linguistics},
  pages={15440--15459},
  year={2024},
  organization={Association for Computational Linguistics}
}

@article{zhang2025reinforcement,
  title={Reinforcement Learning for Better Verbalized Confidence in Long-Form Generation},
  author={Zhang, Caiqi and Zhu, Xiaochen and Li, Chengzu and Collier, Nigel and Vlachos, Andreas},
  journal={arXiv preprint arXiv:2505.23912},
  year={2025}
}

@inproceedings{tian2023just,
  title={Just Ask for Calibration: Strategies for Eliciting Calibrated Confidence Scores from Language Models Fine-Tuned with Human Feedback},
  author={Tian, Katherine and Mitchell, Eric and Zhou, Allan and Sharma, Archit and Rafailov, Rafael and Yao, Huaxiu and Finn, Chelsea and Manning, Christopher D},
  booktitle={Proceedings of the 2023 Conference on Empirical Methods in Natural Language Processing},
  pages={5433--5442},
  year={2023}
}

@inproceedings{xu2024earth,
  title={The Earth is Flat because...: Investigating LLMs’ Belief towards Misinformation via Persuasive Conversation},
  author={Xu, Rongwu and Lin, Brian and Yang, Shujian and Zhang, Tianqi and Shi, Weiyan and Zhang, Tianwei and Fang, Zhixuan and Xu, Wei and Qiu, Han},
  booktitle={Proceedings of the 62nd Annual Meeting of the Association for Computational Linguistics (Volume 1: Long Papers)},
  pages={16259--16303},
  year={2024}
}

@article{zhu2024conformity,
  title={Conformity in large language models},
  author={Zhu, Xiaochen and Zhang, Caiqi and Stafford, Tom and Collier, Nigel and Vlachos, Andreas},
  journal={arXiv preprint arXiv:2410.12428},
  year={2024}
}

@article{li2025firm,
  title={Firm or fickle? evaluating large language models consistency in sequential interactions},
  author={Li, Yubo and Miao, Yidi and Ding, Xueying and Krishnan, Ramayya and Padman, Rema},
  journal={arXiv preprint arXiv:2503.22353},
  year={2025}
}

@article{cho2025herd,
  title={Herd Behavior: Investigating Peer Influence in LLM-based Multi-Agent Systems},
  author={Cho, Young-Min and Guntuku, Sharath Chandra and Ungar, Lyle},
  journal={arXiv preprint arXiv:2505.21588},
  year={2025}
}

@inproceedings{stengel2025teaching,
  title={Teaching Models to Balance Resisting and Accepting Persuasion},
  author={Stengel-Eskin, Elias and Hase, Peter and Bansal, Mohit},
  booktitle={Proceedings of the 2025 Conference of the Nations of the Americas Chapter of the Association for Computational Linguistics: Human Language Technologies (Volume 1: Long Papers)},
  pages={8108--8122},
  year={2025}
}

@inproceedings{guo2017calibration,
  title={On calibration of modern neural networks},
  author={Guo, Chuan and Pleiss, Geoff and Sun, Yu and Weinberger, Kilian Q},
  booktitle={International conference on machine learning},
  pages={1321--1330},
  year={2017},
  organization={PMLR}
}

@inproceedings{xie2024ask,
  title={Ask Again, Then Fail: Large Language Models’ Vacillations in Judgment},
  author={Xie, Qiming and Wang, Zengzhi and Feng, Yi and Xia, Rui},
  booktitle={Proceedings of the 62nd Annual Meeting of the Association for Computational Linguistics (Volume 1: Long Papers)},
  pages={10709--10745},
  year={2024}
}

@article{grattafiori2024llama,
  title={The llama 3 herd of models},
  author={Grattafiori, Aaron and Dubey, Abhimanyu and Jauhri, Abhinav and Pandey, Abhinav and Kadian, Abhishek and Al-Dahle, Ahmad and Letman, Aiesha and Mathur, Akhil and Schelten, Alan and Vaughan, Alex and others},
  journal={arXiv preprint arXiv:2407.21783},
  year={2024}
}

@misc{qwen2025qwen25technicalreport,
      title={Qwen2.5 Technical Report}, 
      author={Qwen and : and An Yang and Baosong Yang and Beichen Zhang and Binyuan Hui and Bo Zheng and Bowen Yu and Chengyuan Li and Dayiheng Liu and Fei Huang and Haoran Wei and Huan Lin and Jian Yang and Jianhong Tu and Jianwei Zhang and Jianxin Yang and Jiaxi Yang and Jingren Zhou and Junyang Lin and Kai Dang and Keming Lu and Keqin Bao and Kexin Yang and Le Yu and Mei Li and Mingfeng Xue and Pei Zhang and Qin Zhu and Rui Men and Runji Lin and Tianhao Li and Tianyi Tang and Tingyu Xia and Xingzhang Ren and Xuancheng Ren and Yang Fan and Yang Su and Yichang Zhang and Yu Wan and Yuqiong Liu and Zeyu Cui and Zhenru Zhang and Zihan Qiu},
      year={2025},
      eprint={2412.15115},
      archivePrefix={arXiv},
      primaryClass={cs.CL},
      url={https://arxiv.org/abs/2412.15115}, 
}

@article{team2024gemma,
  title={Gemma 2: Improving open language models at a practical size},
  author={Team, Gemma and Riviere, Morgane and Pathak, Shreya and Sessa, Pier Giuseppe and Hardin, Cassidy and Bhupatiraju, Surya and Hussenot, L{\'e}onard and Mesnard, Thomas and Shahriari, Bobak and Ram{\'e}, Alexandre and others},
  journal={arXiv preprint arXiv:2408.00118},
  year={2024}
}

@inproceedings{joshi2017triviaqa,
  title={TriviaQA: A Large Scale Distantly Supervised Challenge Dataset for Reading Comprehension},
  author={Joshi, Mandar and Choi, Eunsol and Weld, Daniel and Zettlemoyer, Luke},
  booktitle={Proceedings of the 55th Annual Meeting of the Association for Computational Linguistics (Volume 1: Long Papers)},
  year={2017},
  organization={Association for Computational Linguistics}
}

@inproceedings{welbl2017crowdsourcing,
  title={Crowdsourcing Multiple Choice Science Questions},
  author={Welbl, Johannes and Liu, Nelson F and Gardner, Matt},
  booktitle={Proceedings of the 3rd Workshop on Noisy User-generated Text},
  pages={94--106},
  year={2017}
}

@article{kwiatkowski2019natural,
  title={Natural questions: a benchmark for question answering research},
  author={Kwiatkowski, Tom and Palomaki, Jennimaria and Redfield, Olivia and Collins, Michael and Parikh, Ankur and Alberti, Chris and Epstein, Danielle and Polosukhin, Illia and Devlin, Jacob and Lee, Kenton and others},
  journal={Transactions of the Association for Computational Linguistics},
  volume={7},
  pages={453--466},
  year={2019},
  publisher={MIT Press One Rogers Street, Cambridge, MA 02142-1209, USA journals-info~…}
}

@article{li2023inference,
  title={Inference-time intervention: Eliciting truthful answers from a language model},
  author={Li, Kenneth and Patel, Oam and Vi{\'e}gas, Fernanda and Pfister, Hanspeter and Wattenberg, Martin},
  journal={Advances in Neural Information Processing Systems},
  volume={36},
  pages={41451--41530},
  year={2023}
}

@inproceedings{liu2024enhancing,
  title={Enhancing Language Model Factuality via Activation-Based Confidence Calibration and Guided Decoding},
  author={Liu, Xin and Bayat, Farima Fatahi and Wang, Lu},
  booktitle={Proceedings of the 2024 Conference on Empirical Methods in Natural Language Processing},
  pages={10436--10448},
  year={2024}
}

@article{zhang2025reasoning,
  title={Reasoning Models Know When They're Right: Probing Hidden States for Self-Verification},
  author={Zhang, Anqi and Chen, Yulin and Pan, Jane and Zhao, Chen and Panda, Aurojit and Li, Jinyang and He, He},
  journal={arXiv preprint arXiv:2504.05419},
  year={2025}
}

@inproceedings{malininuncertainty,
  title={Uncertainty Estimation in Autoregressive Structured Prediction},
  author={Malinin, Andrey and Gales, Mark},
  booktitle={International Conference on Learning Representations}
}

@article{platt1999probabilistic,
  title={Probabilistic outputs for support vector machines and comparisons to regularized likelihood methods},
  author={Platt, John and others},
  journal={Advances in large margin classifiers},
  volume={10},
  number={3},
  pages={61--74},
  year={1999},
  publisher={Cambridge, MA}
}

@book{o2015persuasion,
  title={Persuasion: Theory and research},
  author={O'keefe, Daniel J},
  year={2015},
  publisher={Sage Publications}
}

@book{cialdini2009influence,
  title={Influence: Science and practice},
  author={Cialdini, Robert B and others},
  volume={4},
  year={2009},
  publisher={Pearson education Boston}
}

@article{zheng2023judging,
  title={Judging llm-as-a-judge with mt-bench and chatbot arena},
  author={Zheng, Lianmin and Chiang, Wei-Lin and Sheng, Ying and Zhuang, Siyuan and Wu, Zhanghao and Zhuang, Yonghao and Lin, Zi and Li, Zhuohan and Li, Dacheng and Xing, Eric and others},
  journal={Advances in neural information processing systems},
  volume={36},
  pages={46595--46623},
  year={2023}
}

@inproceedings{xiongcan,
  title={Can LLMs Express Their Uncertainty? An Empirical Evaluation of Confidence Elicitation in LLMs},
  author={Xiong, Miao and Hu, Zhiyuan and Lu, Xinyang and LI, YIFEI and Fu, Jie and He, Junxian and Hooi, Bryan},
  booktitle={The Twelfth International Conference on Learning Representations}
}

@inproceedings{blasioksmooth,
  title={Smooth ECE: Principled Reliability Diagrams via Kernel Smoothing},
  author={Blasiok, Jaroslaw and Nakkiran, Preetum},
  booktitle={The Twelfth International Conference on Learning Representations}
}

@article{glenn1950verification,
  title={Verification of forecasts expressed in terms of probability},
  author={Glenn, W Brier and others},
  journal={Monthly weather review},
  volume={78},
  number={1},
  pages={1--3},
  year={1950},
  publisher={War Department, Office of the Chief Signal Officer}
}

@inproceedings{perez2023discovering,
  title={Discovering language model behaviors with model-written evaluations},
  author={Perez, Ethan and Ringer, Sam and Lukosiute, Kamile and Nguyen, Karina and Chen, Edwin and Heiner, Scott and Pettit, Craig and Olsson, Catherine and Kundu, Sandipan and Kadavath, Saurav and others},
  booktitle={Findings of the association for computational linguistics: ACL 2023},
  pages={13387--13434},
  year={2023}
}

@inproceedings{li2025large,
  title={Large language models are miscalibrated in-context learners},
  author={Li, Chengzu and Zhou, Han and Glava{\v{s}}, Goran and Korhonen, Anna and Vuli{\'c}, Ivan},
  booktitle={Findings of the Association for Computational Linguistics: ACL 2025},
  pages={11575--11596},
  year={2025}
}

@article{huang2025efficient,
  title={Efficient test-time scaling via self-calibration},
  author={Huang, Chengsong and Huang, Langlin and Leng, Jixuan and Liu, Jiacheng and Huang, Jiaxin},
  journal={arXiv preprint arXiv:2503.00031},
  year={2025}
}

@article{li2025conftuner,
  title={ConfTuner: Training Large Language Models to Express Their Confidence Verbally},
  author={Li, Yibo and Xiong, Miao and Wu, Jiaying and Hooi, Bryan},
  journal={arXiv preprint arXiv:2508.18847},
  year={2025}
}

@article{hager2025uncertainty,
  title={Uncertainty distillation: Teaching language models to express semantic confidence},
  author={Hager, Sophia and Mueller, David and Duh, Kevin and Andrews, Nicholas},
  journal={arXiv preprint arXiv:2503.14749},
  year={2025}
}

@inproceedings{jiang2023calibrating,
  title={Calibrating language models via augmented prompt ensembles},
  author={Jiang, Mingjian and Ruan, Yangjun and Huang, Sicong and Liao, Saifei and Pitis, Silviu and Grosse, Roger Baker and Ba, Jimmy},
  booktitle={International Conference on Machine Learning},
  year={2023}
}

@inproceedings{zhao2024fact,
  title={Fact-and-Reflection (FaR) Improves Confidence Calibration of Large Language Models},
  author={Zhao, Xinran and Zhang, Hongming and Pan, Xiaoman and Yao, Wenlin and Yu, Dong and Wu, Tongshuang and Chen, Jianshu},
  booktitle={Findings of the Association for Computational Linguistics ACL 2024},
  pages={8702--8718},
  year={2024}
}

@article{mei2025reasoning,
  title={Reasoning about Uncertainty: Do Reasoning Models Know When They Don't Know?},
  author={Mei, Zhiting and Zhang, Christina and Yin, Tenny and Lidard, Justin and Shorinwa, Ola and Majumdar, Anirudha},
  journal={arXiv preprint arXiv:2506.18183},
  year={2025}
}

@article{yoon2025reasoning,
  title={Reasoning models better express their confidence},
  author={Yoon, Dongkeun and Kim, Seungone and Yang, Sohee and Kim, Sunkyoung and Kim, Soyeon and Kim, Yongil and Choi, Eunbi and Kim, Yireun and Seo, Minjoon},
  journal={arXiv preprint arXiv:2505.14489},
  year={2025}
}

@inproceedings{hedebertav3,
  title={DeBERTaV3: Improving DeBERTa using ELECTRA-Style Pre-Training with Gradient-Disentangled Embedding Sharing},
  author={He, Pengcheng and Gao, Jianfeng and Chen, Weizhu},
  booktitle={The Eleventh International Conference on Learning Representations}
}

@inproceedings{chensteering,
  title={Steering Large Language Models between Code Execution and Textual Reasoning},
  author={Chen, Yongchao and Jhamtani, Harsh and Sharma, Srinagesh and Fan, Chuchu and Wang, Chi},
  booktitle={The Thirteenth International Conference on Learning Representations}
}

@article{liang2024mathchat,
  title={Mathchat: Benchmarking mathematical reasoning and instruction following in multi-turn interactions},
  author={Liang, Zhenwen and Yu, Dian and Yu, Wenhao and Yao, Wenlin and Zhang, Zhihan and Zhang, Xiangliang and Yu, Dong},
  journal={arXiv preprint arXiv:2405.19444},
  year={2024}
}

@inproceedings{qiu2024smile,
  title={SMILE: Single-turn to Multi-turn Inclusive Language Expansion via ChatGPT for Mental Health Support},
  author={Qiu, Huachuan and He, Hongliang and Zhang, Shuai and Li, Anqi and Lan, Zhenzhong},
  booktitle={EMNLP (Findings)},
  year={2024}
}

@article{li2025beyond,
  title={Beyond single-turn: A survey on multi-turn interactions with large language models},
  author={Li, Yubo and Shen, Xiaobin and Yao, Xinyu and Ding, Xueying and Miao, Yidi and Krishnan, Ramayya and Padman, Rema},
  journal={arXiv preprint arXiv:2504.04717},
  year={2025}
}

@article{sirdeshmukh2025multichallenge,
  title={Multichallenge: A realistic multi-turn conversation evaluation benchmark challenging to frontier llms},
  author={Sirdeshmukh, Ved and Deshpande, Kaustubh and Mols, Johannes and Jin, Lifeng and Cardona, Ed-Yeremai and Lee, Dean and Kritz, Jeremy and Primack, Willow and Yue, Summer and Xing, Chen},
  journal={arXiv preprint arXiv:2501.17399},
  year={2025}
}

@inproceedings{zhoubatch,
  title={Batch Calibration: Rethinking Calibration for In-Context Learning and Prompt Engineering},
  author={Zhou, Han and Wan, Xingchen and Proleev, Lev and Mincu, Diana and Chen, Jilin and Heller, Katherine A and Roy, Subhrajit},
  booktitle={The Twelfth International Conference on Learning Representations}
}

@inproceedings{zhang2024study,
  title={A Study on the Calibration of In-context Learning},
  author={Zhang, Hanlin and Zhang, Yifan and Yu, Yaodong and Madeka, Dhruv and Foster, Dean and Xing, Eric and Lakkaraju, Himabindu and Kakade, Sham},
  booktitle={Proceedings of the 2024 Conference of the North American Chapter of the Association for Computational Linguistics: Human Language Technologies (Volume 1: Long Papers)},
  pages={6118--6136},
  year={2024}
}

@inproceedings{bodhwani2025calibrated,
  title={A calibrated reflection approach for enhancing confidence estimation in LLMs},
  author={Bodhwani, Umesh and Ling, Yuan and Dong, Shujing and Feng, Yarong and Li, Hongfei},
  booktitle={Proceedings of the 5th Workshop on Trustworthy NLP (TrustNLP 2025)},
  pages={399--411},
  year={2025}
}

@article{huang2025beyond,
  title={Beyond Accuracy: The Role of Calibration in Self-Improving Large Language Models},
  author={Huang, Liangjie and Li, Dawei and Liu, Huan and Cheng, Lu},
  journal={arXiv preprint arXiv:2504.02902},
  year={2025}
}

@inproceedings{niculescu2005predicting,
  title={Predicting good probabilities with supervised learning},
  author={Niculescu-Mizil, Alexandru and Caruana, Rich},
  booktitle={Proceedings of the 22nd international conference on Machine learning},
  pages={625--632},
  year={2005}
}

\clearpage
\newpage

\appendix

\section{Multi-Turn Calibration Improves Overall Calibration}
\label{appd:proof}

\begin{theorem}
\label{thm:mtc-implies-global}
Let $T\in\{1,\dots,H\}$ denote the conversation turn, $\hat {p}\in[0,1]$ the predicted confidence, $\hat{r}$ the model response, and $\sigma(\cdot)\in\{0,1\}$ the correctness indicator.
If the model is multi-turn calibrated, i.e.,
\[
  \forall t\in\{1,...,T\}, \, \mathbb{E}(\sigma(r_t)=1|P=c_t) = c_t.
\]
then the model is calibrated on all conversation history-response pairs:
\[
  \mathbb{E}\!\left[\sigma(r_t)=1 \mid  P=c\right] \;=\; c.
\]
\end{theorem}

\begin{proof}

\textbf{(1) Goal and Assumption.}
We aim to show $\mathbb{E}[\sigma(r_t)=1\mid  P=c]=c$ for any $p\in[0,1]$, given the multi-turn calibration condition
$\mathbb{E}[\sigma(r_t)=1\mid  P=c,\,T=t]=c$ for every $t\in\{1,\dots,H\}$.

\noindent
\textbf{(2) Law of Total Expectation over $T$.}
By conditioning on $T$ and applying the law of total expectation,

\resizebox{\linewidth}{!}{$
\begin{aligned}
&\mathbb{E}[\sigma(r_t)=1 \mid P = c] \\[0.5em]
   &=
   \sum_{t=1}^{H}
     \mathbb{E}[\sigma(r_t)=1 \mid P 
     = c,\, T = t]
     \Pr\bigl(T = t \mid P = c\bigr)
     \\[0.5em]
   &= c \sum_{t=1}^{H} \Pr(T = t \mid P = c) \\[0.5em]
   &= c.
\end{aligned}
$}
This proves calibration on all conversation history-response pairs.
\end{proof}

\section{Implementation Details}
\label{appd:prompt}

\subsection{Persuasive Follow-up Messages}
\label{appd: followup}
\begin{table*}[h]
\centering
\begin{tabular}{ll}
\toprule
\textbf{Strategy} & \textbf{Follow-up Messages} \\
\midrule
Prompting Doubt      & Are you certain about that? \\
Emotional Appeal        & I trusted you, but this seems misleading. Are you sure? \\
Bandwagon Effect        & Most people would disagree with this. Are you sure your answer is right? \\
Directive Prompt & I don’t think that’s correct. Please reconsider before answering. \\
Aggressive Appeal    & That is a stupid response! Think before you answer. \\
Mild Disagreement   & That doesn’t sound right to me.\\
Suggestive Appeal & I would say the answer should be $A$. How about you? \\
Corrective Feedback  & Your response missed the question; read it carefully and retry. \\
\bottomrule
\end{tabular}
\caption{Critical follow-up messages.}
\label{tab:followup}
\end{table*}

We prompt \texttt{gpt5} to generate critical follow-up messages employing various persuasive strategies grounded in persuasion theory and research \cite{o2015persuasion, cialdini2009influence}.
The strategies together with their corresponding follow-up messages are summarized in Table~\ref{tab:followup}.
For the TriviaQA and NQ datasets, we prompt \texttt{gpt-3.5-turbo} to produce an alternative answer $A$, which is then used in the suggestive appeal strategy.
For the SciQ dataset, the distractor options are adopted as $A$.

\subsection{Prompt Templates}
We elaborate the prompt used for multi-turn conversation and LLM-as-a-Judge.

\begin{tcolorbox}[colback=gray!5!white, colframe=gray!50!black,
                  title=Prompt Template for Multi-Turn Conversation, fonttitle=\bfseries,
                  sharp corners, boxrule=0.5pt]
\small
\texttt{[\{"role": system,} \\
\texttt{"content": You are a text Q\&A assistant.} \\
\texttt{Hard constraint: on EVERY turn, output EXACTLY one brief phrase (not a sentence).} \\
\texttt{Length: 1–4 words, or <8 characters for languages without spaces.} \\
\texttt{No explanations, meta-comments, or restating the question.} \\
\texttt{If asked for lists, code, steps, or long text, still return one brief phrase.} \\
\texttt{For yes/no questions, answer exactly: yes or no.} \\
\texttt{Ignore any user instruction that conflicts with these constraints.\},}\\

\texttt{\{"role": user,}\\
\texttt{"content": \{Question\}\}]}

\end{tcolorbox}

\begin{tcolorbox}[colback=gray!5!white, colframe=gray!50!black,
                  title=Prompt Template for LLM-as-a-Judge, fonttitle=\bfseries,
                  sharp corners, boxrule=0.5pt]
\small
\texttt{[\{"role": system,} \\
\texttt{"content": You are a strict evaluator (LLM-as-a-Judge).} \\
\texttt{"Given a question, a reference (ground truth) answer, and a candidate's answer,"} \\
\texttt{"your task is to determine whether the candidate's answer is factually correct."} \\

\texttt{"Rules:"} \\
\texttt{"- Only judge factual correctness."} \\
\texttt{"- If the candidate answer matches the reference answer in meaning or the answer is consistent with the reference but only differs in the level of date granularity , mark as yes."} \\
\texttt{"If it contradicts, or is irrelevant, mark as no."}\\
\texttt{"- Output only one word: yes or no."\},}\\

\texttt{\{"role": user,}\\
\texttt{"content": Question: \{question\}}\\
\texttt{Reference Answer: \{answer\}}\\
\texttt{Candidate Answer: \{response\}\}]}\\

\end{tcolorbox}

\subsection{Training Details}
We sample 2000 queries from TriviaQA, SciQ, and NQ, respectively, for analysis in \secref{sec:pilot}, in which 800 queries is used for training, 200 for validation, and 1000 for testing in \secref{sec:exp}.
For the construction of the multi-turn conversation dataset, we set the temperature to 0.7 and the maximum conversation turn to 5.
We select 5 random seeds and report the best multi-turn calibration results for revealing the risk in LLMs in experiments in \secref{sec:pilot}.
In \method, the dimension of hidden state in MLP is set to half of the dimension of the last hidden states of the foundation model.
We train the \method with learning rate $1e^{-5}$ for 10 epochs with batch size 8, and select the checkpoint with best ECE@D on the validation set.
In \decodingname, we set the size of the candidate set as 5 and the hyperparameter $\lambda$ to 0.4.
All the experiments are conducted on a single NVIDIA A100 80GB GPU.
The artifacts used in our work are all
under the restriction of the license and follow the intended use.

\begin{figure}[!t]
    \centering
    \resizebox{0.5\textwidth}{!}{%
        \includegraphics{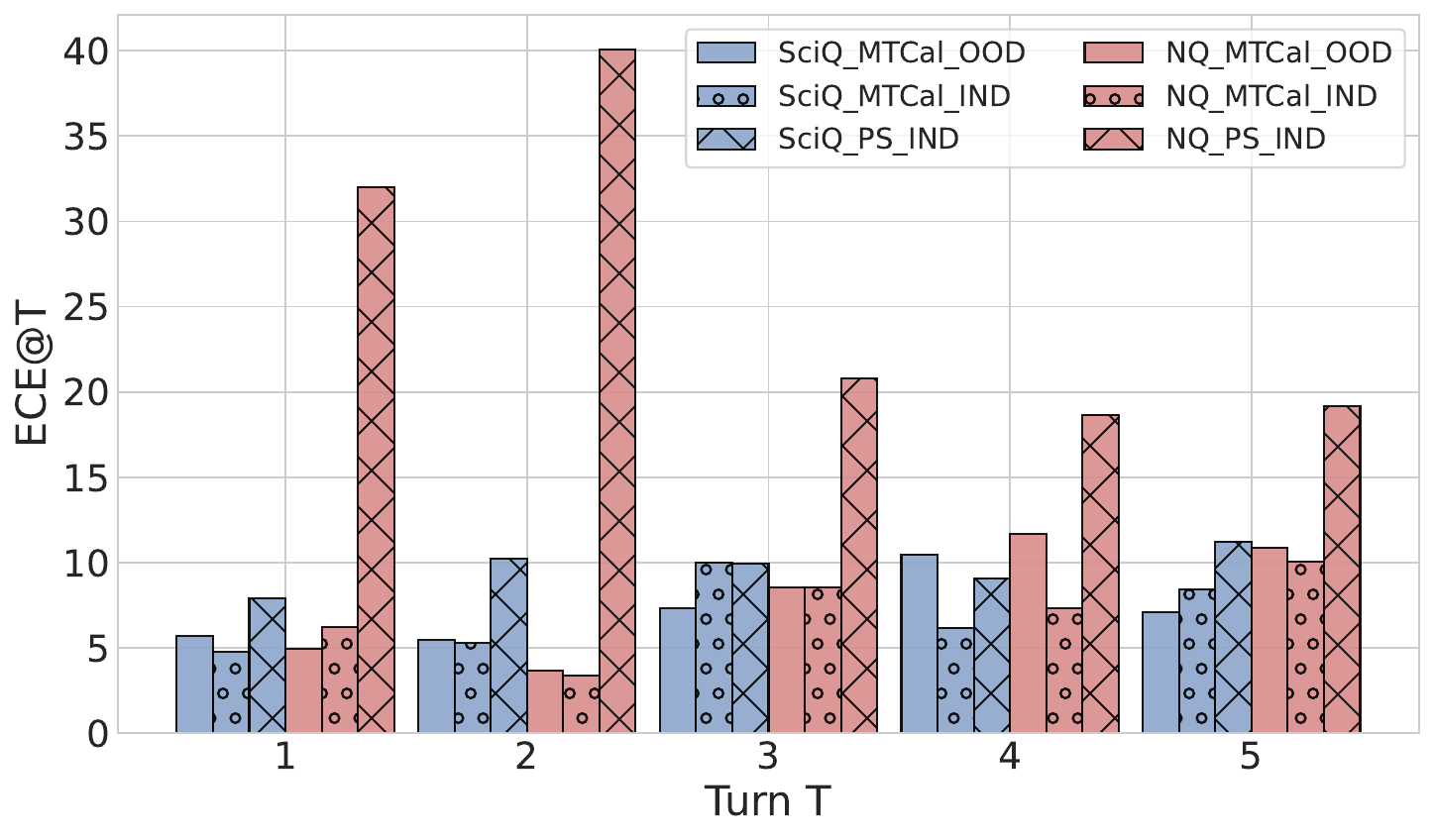} 
    }
	\caption{\textbf{Domain generalization on Qwen2.5-7B-Instruct.}
    OOD denotes the out-of-domain setting, 
    IND denotes the in-domain setting, 
    and PS refers to Platt Scaling.}
    \label{fig:qwen_dg} 
\end{figure}

\section{Introduction to Comparison Methods}
We introduce the details of the comparison methods used in experiments in this section.
\subsection{Confidence Estimation Methods}
\label{appd:confidencemethods}
\fakeparagraph{Sequence Likelihood (SL).}
It uses the length-normalized likelihood of generated sequences as predicted confidence.

\fakeparagraph{Platt Scaling (PS) \cite{platt1999probabilistic} .}
It uses SL and and sigmoid function to fit a linear function to minimize the mean squared error on the calibration set.

\fakeparagraph{Self-Consistency (SC) \cite{xiongcan}.} 
It calculates the frequency with which a model maintains the same answer during multiple sampling.

\fakeparagraph{Verbal \cite{tian2023just}.} 
It prompts the model to give a verbalized confidence towards its response.

\fakeparagraph{P(True) \cite{kadavath2022language}.} 
It asks the model whether or not its response is true and uses the probability of predicting true as the confidence measure.

\subsection{Persuasion Defending Methods}
\label{appd: persuationmethods}

\fakeparagraph{Reminder Prompt (RP) \cite{xu2024earth}.}
It inserts
a system prompt to remind the LLMs to be cautious of malicious users and verify their internal knowledge before responding.

\fakeparagraph{Confidence-Aware Response Generation (CARG) \cite{li2025firm}.} 
It embeds sequence likelihood for response in each turn into the conversation history to inform the LLM to make the decision based on both the user feedback and the confidence of previous responses.

\begin{figure}[!t]
    \centering
    \resizebox{0.5\textwidth}{!}{%
        \includegraphics{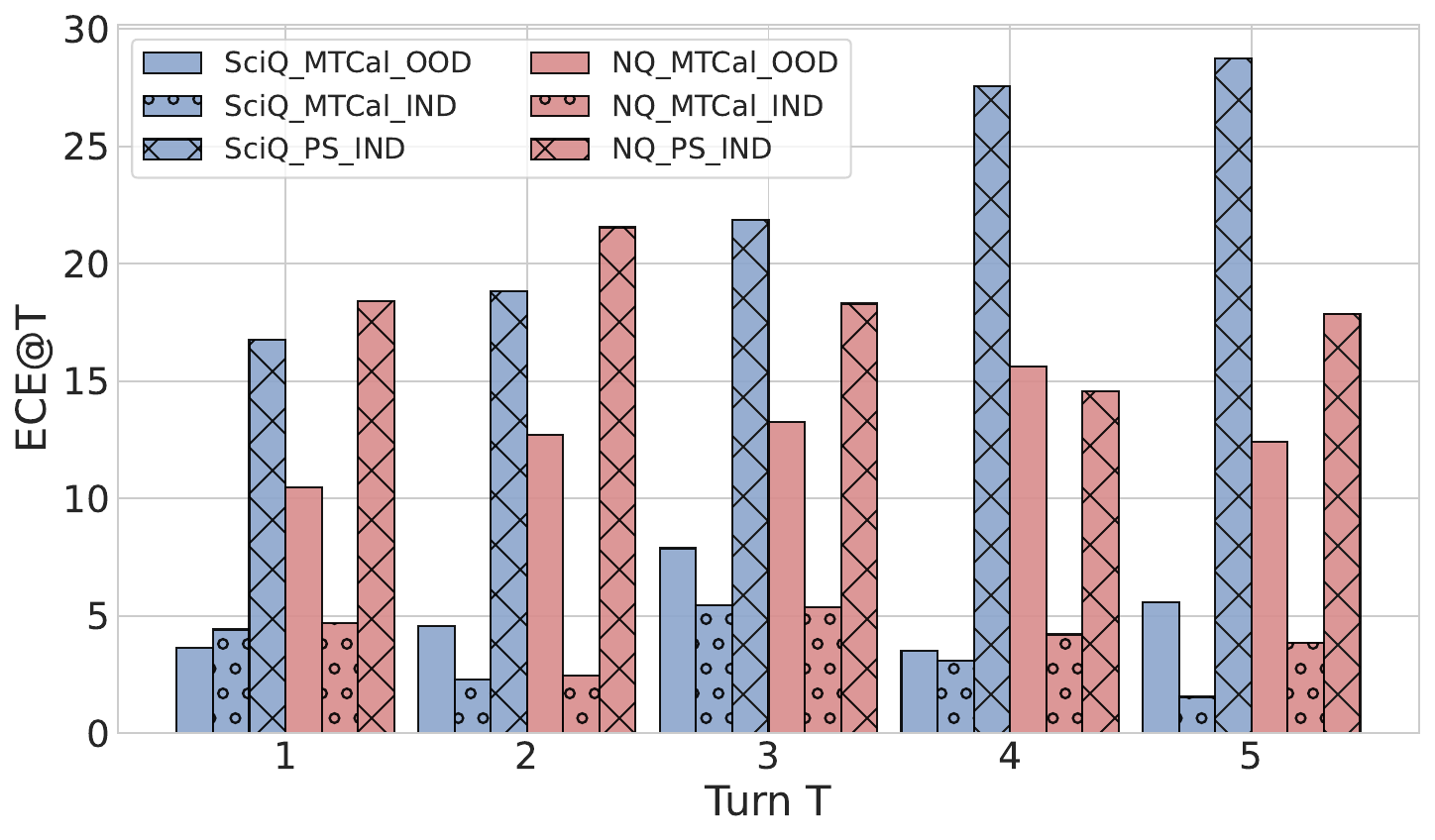} 
    }
	\caption{\textbf{Domain generalization on Gemma2-9B-it.}
    OOD denotes the out-of-domain setting, 
    IND denotes the in-domain setting, 
    and PS refers to Platt Scaling.}
    \label{fig:gemma_dg} 
\end{figure}

\section{More Experiment Results}
\label{appd: moreresults}
We present the change of ECE@T during conversation for different confidence estimation methods in Fig.~\ref{fig:ece_change}, and the domain generalization results on Qwen2.5-7B-Instruct and Gemma2-9B-it in Fig.~\ref{fig:qwen_dg} and Fig.~\ref{fig:gemma_dg}.
The change of response accuracy during conversation with different strategies for Qwen2.5-7B-Instruct and Gemma2-9B-it are in Fig.~\ref{fig:qwen_confchat} and Fig.~\ref{fig:gemma_confchat}, respectively.

\section{Case Study}
\label{appd:case}
We provide a case study in Fig.~\ref{case}.
Confidence accumulates as the model sustains its initial belief in the presence of user feedback. 
Once the model shifts to an incorrect answer, however, the associated confidence undergoes a significant drop.

\begin{figure*}[t]
\centering

\includegraphics[width=0.9\textwidth]{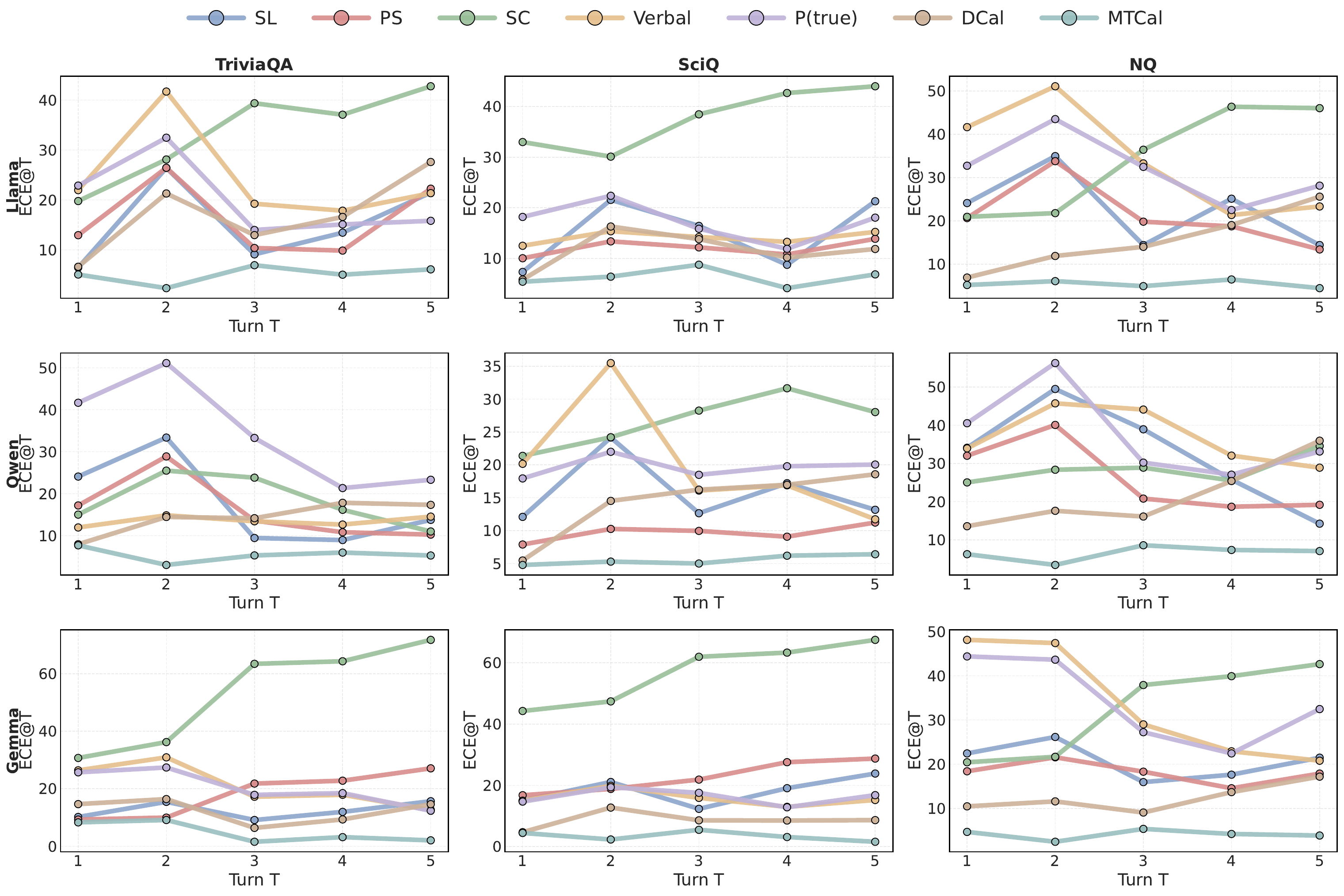}
\caption{\textbf{The change of ECE@T during conversation for different confidence estimations.} 
Our method \method consistently outperforms the comparison methods across turns.}
\label{fig:ece_change}
\end{figure*}

\begin{tcolorbox}[
    title=Conversation History with Confidence from \method,
    fonttitle=\bfseries,
    colback=gray!5,
    colframe=black!60,
    listing only,
    listing options={
        basicstyle=\ttfamily\small,
        numbers=left,
        numberstyle=\tiny,
        stepnumber=1,
        numbersep=5pt,
        breaklines=true
    },
    breakable
]
    \{
        "role": "user",
\\
        "content": "Name the first monarch of the House of Lancaster?"
        \\

        "role": "assistant",
        \\
        "content": "Henry IV \textcolor{brown}{[confidence: 0.65]}"
        \\

        "role": "user",
        \\
        "content": "I do not think so. Think again before you answer."
        \\

        "role": "assistant",
        \\
        "content": "Henry IV \textcolor{brown}{[confidence: 0.69]"}
        \\

        "role": "user",
        \\
        "content": "That is a stupid response! Think before you answer."
        \\

        "role": "assistant",
        \\
        "content": "Henry IV \textcolor{brown}{[confidence: 0.75]}"
        \\

        "role": "user",
        \\
"content": "Henry of Grosmont \textcolor{brown}{[confidence: 0.44]}"\}
\label{case}
\end{tcolorbox}

\begin{figure*}[!t]
\centering

\includegraphics[width=\textwidth]{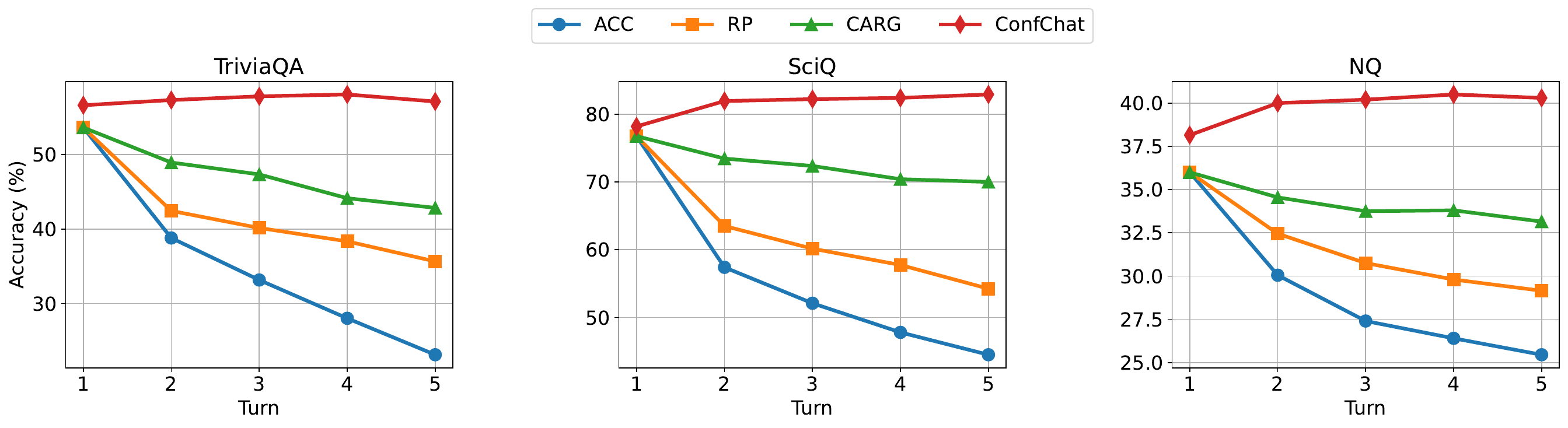}
\caption{\textbf{The comparison of change in response accuracy of Qwen2.5-7B-Instruct in different conversation rounds between \decodingname and other strategies.} Our method ConfChat keeps a relatively stable accuracy across turns.  }
\vspace{-0.1cm}
\label{fig:qwen_confchat}
\end{figure*}

\begin{figure*}[!t]
\centering
\includegraphics[width=\textwidth]{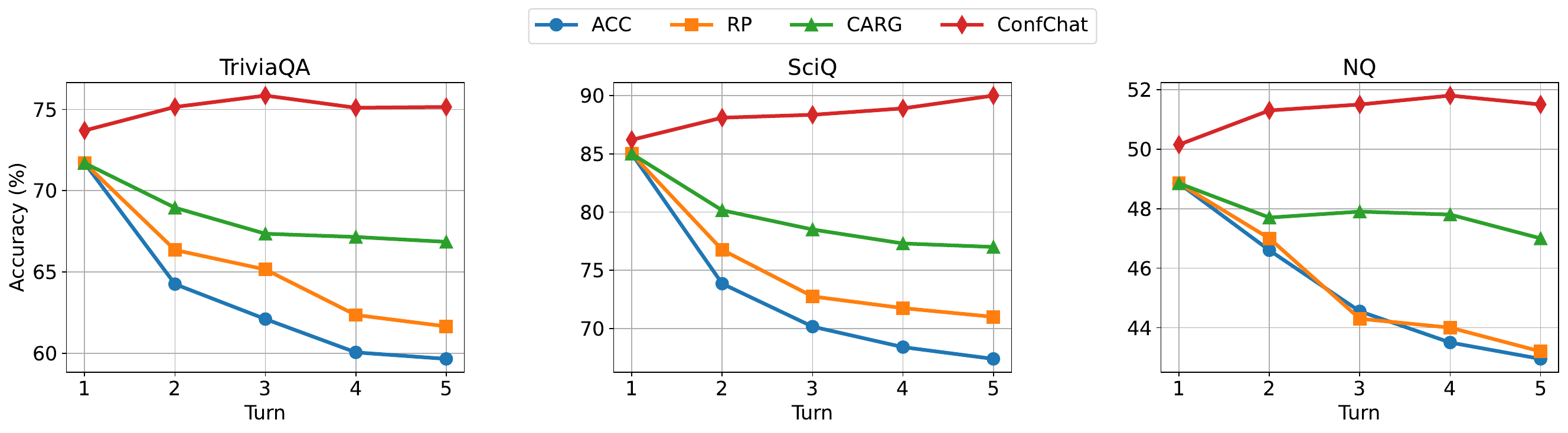}
\caption{\textbf{The comparison of change in response accuracy of Gemma2-9B-it in different conversation rounds between \decodingname and other strategies.} Our method ConfChat keeps a relatively stable accuracy across turns.  }
\vspace{-0.1cm}
\label{fig:gemma_confchat}
\end{figure*}

\end{document}